\pdfoutput=1

\documentclass[11pt,a4paper]{article}
\usepackage{times,latexsym}
\usepackage{url}
\usepackage[T1]{fontenc}

\usepackage[acceptedWithA]{tacl2021v1} 

\usepackage{xspace,mfirstuc,tabulary}

\usepackage{booktabs}

\usepackage[T1]{fontenc}
\usepackage[utf8]{inputenc}

\usepackage{microtype}

\usepackage{graphicx}
\usepackage{amsfonts}
\usepackage{amsmath}
\usepackage{caption}
\usepackage{subcaption}
\usepackage{bbm}
\usepackage{xspace}
\usepackage{multirow}

\usepackage{amssymb}%
\usepackage{pifont}%
\newcommand{\cmark}{\ding{51}}%
\newcommand{\xmark}{\ding{55}}%

\usepackage{enumerate}
\usepackage{enumitem}
\usepackage{amsthm}
\usepackage{xfrac}

\usepackage{inconsolata}

\usepackage{cleveref}
\crefname{section}{\S}{\S\S}
\Crefname{section}{\S}{\S\S}
\crefformat{section}{\S#2#1#3}
\crefname{figure}{Fig.}{Fig.}
\crefname{alg}{Alg.}{Alg.}
\crefname{line}{line}{lines}
\crefname{appendix}{App.}{}
\crefname{equation}{eq.}{eqs.}
\crefname{table}{Table}{Tables}
\crefname{prop}{Proposition}{Propositions}

\newcommand{\defeq}[0]{\mathrel{\stackrel{\textnormal{\tiny def}}{=}}}

\newtheorem{theorem}{Theorem}
\newtheorem{proposition}{Proposition}

\newcommand{\word}{w}
\newcommand{\words}{\boldsymbol{w}}
\newcommand{\Word}{W}
\newcommand{\Words}{\boldsymbol{W}}
\newcommand{\prevwords}{\words_{< t}}

\newcommand{\surpfunc}{h}
\newcommand{\entfunc}{\mathrm{H}}

\newcommand{\surp}{\surpfunc_t(\word_t)}
\newcommand{\surpprev}{\surpfunc_{t\!-\!1}(\word_{t\!-\!1})}
\newcommand{\ent}{\entfunc(W_t)}
\newcommand{\entprev}{\entfunc(W_{t-1})}
\newcommand{\entprevprev}{\entfunc(W_{t-2})}
\newcommand{\entprevprevprev}{\entfunc(W_{t-3})}

\newcommand{\deltallh}{\Delta_{\mathrm{llh}}}
\newcommand{\vocabnoeos}{\mathcal{W}}
\newcommand{\vocab}{\overline{\vocabnoeos}}
\newcommand{\renyientfunc}{\entfunc_{\alpha}}

\newcommand{\renyient}{\renyientfunc\!\left(W_{t}\right)}

\newcommand{\renyientnext}{\renyientfunc(W_{t+1})}
\newcommand{\timefunc}{y}
\newcommand{\skipratio}{y}
\newcommand{\ReadingMeasure}{y}
\newcommand{\bx}{\mathbf{x}}
\newcommand{\gamparams}{\boldsymbol{\phi}}
\newcommand{\predfunc}{f_{\gamparams}}
\newcommand{\unigramfunc}{u}
\newcommand{\unigram}{\unigramfunc(\word_t)}

\newcommand{\bxbase}{\bx^{\mathrm{base}}}

\newcommand{\btheta}{\boldsymbol{\theta}}
\newcommand{\ptheta}{p_{\btheta}}

\newcommand{\psupport}{\mathrm{supp}(p)}

\newcommand{\eos}{\ensuremath{\textsc{eos}}\xspace}

\newcommand{\enti}{\entfunc(\Word_{t'})}
\newcommand{\renyienti}{\renyientfunc(\Word_{t'})}

\newcommand{\R}{\mathbb{R}}
\DeclareMathOperator*{\expect}{\mathbb{E}}

\newcommand{\defn}[1]{\textbf{#1}}
\newcommand{\citeposs}[1]{\citeauthor{#1}'s (\citeyear{#1})}

\newcommand{\renyi}{R\'enyi\xspace}

\newcommand{\llh}{\mathrm{llh}}
\newcommand{\bxmodel}{\bx^{\mathrm{model}}}
\newcommand{\dataset}{\mathcal{D}}

\newcommand{\gpt}{GPT-2\xspace}
\newcommand{\gptsmall}{GPT-2 \texttt{small}\xspace}

\newcommand{\bxcommon}{\bx^{\mathrm{cmn}}}
\newcommand{\bxsurp}{\bx^{\mathrm{surp}}}
\newcommand{\bxsurpnott}{\bx^{\mathrm{surp} \neq t}}
\newcommand{\bxsurpnoti}{\bx^{\mathrm{surp} \neq t'}}

\definecolor{mygray}{rgb}{.3, .3, .3}
\definecolor{mygreen}{rgb}{0, .5, 0}
\definecolor{myred}{rgb}{.6, 0.15, 0.15}

\newcommand{\trimheight}{.4cm}

\title{
On the Effect of Anticipation on Reading Times
}

\usepackage{emoji}
\newcommand{\ucambridge}{\emoji[emoji]{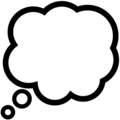}}
\newcommand{\ethz}{\emoji[emoji]{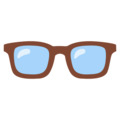}}
\newcommand{\mituni}{\emoji[emoji]{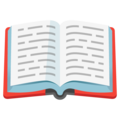}}

\author{
 Tiago Pimentel$^{\ucambridge}$~
 Clara Meister$^{\ethz}$~
 Ethan G. Wilcox$^{\ethz}$~
 \textbf{Roger P.\ Levy}$^{\mituni}$~
 \textbf{Ryan Cotterell}$^{\ethz}$ \\
 $^{\ucambridge}$University of Cambridge~\;~  $^{\ethz}$ETH Zürich~\;~  $^{\mituni}$MIT\\
 \texttt{\href{mailto:tp472@cam.ac.uk}{tp472@cam.ac.uk}}~\;~ 
 \texttt{\href{mailto:clara.meister@inf.ethz.ch}{clara.meister@inf.ethz.ch}}~\;~ 
 \texttt{\href{mailto:ethan.wilcox@inf.ethz.ch}{ethan.wilcox@inf.ethz.ch}} \\ \texttt{\href{mailto:rplevy@mit.edu}{rplevy@mit.edu}}~\;~  \texttt{\href{mailto:ryan.cotterell@inf.ethz.ch}{ryan.cotterell@inf.ethz.ch}}
}

\begin{document}
\maketitle
\begin{abstract}
Over the past two decades, numerous studies have demonstrated how less predictable (i.e., higher surprisal) words take more time to read.
In general, these studies have implicitly assumed the reading process is purely \emph{responsive}: Readers observe a new word and allocate time to process it as required. 
We argue that prior results are also compatible with a reading process that is at least partially \emph{anticipatory}: Readers could make predictions about a future word and allocate time to process it based on their expectation.
In this work, we operationalize this anticipation as a word's contextual entropy.
We assess the effect of anticipation on reading by comparing how well surprisal and contextual entropy predict reading times on four naturalistic reading datasets: two self-paced and two eye-tracking. 
Experimentally, across datasets and analyses, we find substantial evidence for effects of contextual entropy over surprisal on a word's reading time (RT): in fact, entropy is sometimes better than surprisal in predicting a word's RT. 
Spillover effects, however, are generally not captured by entropy, but only by surprisal.
Further, we hypothesize four cognitive mechanisms through which contextual entropy could impact RTs---three of which we are able to design experiments to analyze.
Overall, our results support a view of reading that is not just responsive, but also anticipatory.%
\footnote{Code is available at \url{https://github.com/rycolab/anticipation-on-reading-times}.}\looseness=-1 
\end{abstract}

\vspace{-15pt}
\section{Introduction}

Language comprehension---and, by proxy, the reading process---is assumed to be incremental and dynamic in nature: Readers take in one word at a time, process it, and then move on to the next word \cite{hale2001probabilistic,hale2006uncertainty,rayner2009language,boston2011parallel}.
As each word requires a different amount of processing effort, readers must dynamically allocate differing amounts of time to words as they read.
Indeed, this effect has been confirmed by a number of studies, which show a word's reading time is a monotonically increasing function of the word's length and surprisal \citep[\textit{inter alia}]{hale2001probabilistic,smith2008optimal,shain-2019-large}.\looseness=-1

Most prior work \cite[e.g.,][]{levy2008expectation,DEMBERG2008193,monsalve2012,wilcox2020predictive}, however, focuses on the \defn{responsive} nature of the reading process, i.e., prior work looks solely at how a reader's behavior is influenced by attributes of words which have already been observed.
Such analyses make the assumption that readers dynamically allocate resources to predict future words' identities in advance, but that the distribution of those predictions do not themselves directly affect reading behavior.
However, a closer analysis of RT data shows the above theory might not capture the whole picture.
In addition to being responsive, reading behavior may also be \defn{anticipatory}: 
Readers' predictions may influence reading behavior for a word regardless of its actual identity.\looseness=-1

Theoretically, anticipatory reading behavior may be an adaptive response to oculomotor constraints, as it takes time both to identify a word and to program a motor response to move beyond it.
An example of anticipatory behavior is that the eyes often skip over words while reading---a decision that must be made while the word's identity remains uncertain \cite{ehrlich-rayner:1981,schotter-etal:2012-parafoveal}.
We identify four mechanisms that are anticipatory in nature and may impact reading behaviors:\looseness=-1
\vspace{-3pt}
\begin{enumerate}[label=(\roman*),noitemsep]
\item \textbf{word skipping}: readers may completely omit fixating on a word;\looseness=-1  
\item \textbf{budgeting}: readers may pre-allocate RTs for a word before reaching it; 
\item \textbf{preemptive processing}: readers may start processing a future word based on their expectations (and before knowing its identity);\looseness=-1
\item \textbf{uncertainty cost}:
readers may incur an additional processing load when in high uncertainty contexts.
\end{enumerate}
\vspace{-3pt}

In this work, we look beyond responsiveness, investigating anticipatory reading behaviors and the mechanisms above.
Specifically, we look at how a reader's expectation about a word's surprisal---operationalized as that word's \defn{contextual entropy}---affects the time taken to read it. 
For various reasons, however, a reader's anticipation may not exactly match a word's expected surprisal value, which would make the contextual entropy a poor operationalization of anticipation.
Rather, readers may rely on skewed approximations instead, e.g., anticipating that the next word's surprisal is simply the surprisal of the most likely next word.
We use the \renyi{} entropy (a generalization of Shannon's entropy) to operationalize these different skewed expectation strategies.
We then design several experiments to investigate the mechanisms above, analyzing the relationship between readers' expectations about a word's surprisal and its observed RTs.\looseness=-1

We run our analyses in four naturalistic datasets: two self-paced reading and two eye-tracking.
In line with prior work, we find a significant effect of a word's surprisal on its RTs across all datasets, reaffirming the responsive nature of reading.
In addition, we find the word's contextual entropy to be a significant predictor of its RTs in three of the four analyzed datasets---in fact, in two of these, entropy is a more powerful predictor than  surprisal; see \cref{tab:delta_llh_entropy_renyi}.
Unlike surprisal however, in general, entropy does not predict spillover effects.
We further find that a specific \renyi{} entropy (with $\alpha=\sfrac{1}{2}$) consistently leads to stronger predictors than the Shannon entropy. 
Our finding suggests readers may anticipate a future word's surprisal to be a function of the number of plausible word-level continuations (as opposed to the actual expected surprisal).\looseness=-1

\section{Predicting Reading Behavior} \label{sec:predictive_rt}

One behavior of interest in psycholinguistics is reading time (RT) allocation, i.e., how much time a reader spends processing each word in a text.
RTs and other eye movement measures, such as word skipping (\cref{sec:word-skip_planning}), are important for psycholinguistics because they offer insights into the mechanisms driving the reading process.
Indeed, there exists a vast literature of such analyses
\citep[\textit{inter alia}]{rayner1998eye,hale2001probabilistic,hale2003information,hale2016information,keller-2004-entropy,van-schijndel-linzen-2018-neural,shain-2019-large,shain-2021-cdrnn,shain2021continuous,shain2022deep,wilcox2020predictive,meister-etal-2021-revisiting,meister-etal-2022-analyzing,kuribayashi-etal-2021-lower,kuribayashi-etal-2022-context,hoover2022plausibility}.\footnote{For more comprehensive introductions to computational reading time analyses see \citet{rayner1998eye,rayner2005eye}.}

The standard procedure for  reading behavior analysis is to first choose a set of variables $\bx \in \R^d$ which is believed to impact reading---e.g., we could choose $\bx_t=[|\word_t|, \unigram]^\intercal$,
where $|\word_t|$ is the length of word $\word_t$ and $\unigram$ is its frequency (quantified as its unigram log-probability).
These variables are then used to fit a regressor $\predfunc(\bx)$ of a reading measure $\ReadingMeasure$:\looseness=-1%
\begin{equation} \label{eq:regressor}
    \ReadingMeasure (w_t \mid \prevwords) \sim \predfunc(\bx_t)
\end{equation}
where $\predfunc: \R^d \rightarrow \R$, $\gamparams$ are learned parameters, and $\ReadingMeasure$ will be either reading times or word skipping ratio here.
We then evaluate this regressor by looking at its performance, which is typically operationalized as the average log-likelihood assigned by $\predfunc(\bx)$ to held out data \cite{goodkind-bicknell-2018-predictive,wilcox2020predictive}.

When comparing different theories of the reading process, each may predict a different architecture $\predfunc$ or set of variables $\bx$ which should be used in \cref{eq:regressor}.
We can then compare these theories by looking at the performance of their associated regressors. 
Specifically, we take a model that leads to higher log-likelihoods on held out data as evidence in favor of its corresponding theory about underlying cognitive mechanisms. 
Further, model $\predfunc(\bx)$ can then be used to understand the relationship between the employed predictors $\bx$ and RTs.\looseness=-1

\subsection{Responsive Reading}

One of the most studied variables in the above paradigm is \defn{surprisal}, which measures a word's information content.
Surprisal theory \cite{hale2001probabilistic,levy2008expectation} posits that a word's surprisal should directly impact its processing cost. 
Intuitively, this makes sense: The higher a word's information content, the more resources it should take to process that word. 
Surprisal theory has since sparked a line of research exploring the relationship between surprisal and processing cost, where a word's processing cost is typically quantified as its RT.\footnote{\citet{levy_thesis}, for instance, showed that a word's surprisal quantifies a change in the reader's belief over sentence continuations. He then posited this change in belief may be reflected as processing cost.\looseness=-1}\looseness=-1

Formally, the surprisal (or information content) of a word is defined as its in-context negative log-likelihood \citep{shannon1948mathematical}, which we denote as%
\begin{subequations}\label{eq:surprisal}
\begin{align} 
	\surpfunc_t(\word) &\defeq 
 \entfunc(\Word_t = \word \mid \Words_{<t} = \prevwords)\\
 &= -\log_2 p(\word \mid \words_{< t}) 
\end{align}
\end{subequations}
where $p$ is the ground-truth probability distribution over natural language utterances. 
We resort to $h_t(w)$ as a convenient shorthand that avoids notational clutter.
In words, \cref{eq:surprisal} states that a word is more surprising---and thus conveys more information---if it is less likely, and vice versa.\looseness=-1

Time and again, the surprisal has proven to be a strong predictor in RT analyses \citep[\emph{inter alia}]{smith2008optimal,smith2013-log-reading-time,goodkind-bicknell-2018-predictive,wilcox2020predictive,shain2022large}. 
Importantly, surprisal (as well as other properties of a word, like frequency or length) is a quantity that can only feasibly impact readers' behaviors after they have encountered the word in question.\footnote{This follows from standard theories of causality. \citet{granger1969investigating}, for instance, posits that future material cannot influence present behavior.}
Thus, by limiting their analyses to such characteristics, these prior works assume RT allocation happens \emph{after} word identification, being thus \defn{responsive} to the context a reader finds themselves in and happening on demand as needed for processing a word.\looseness=-1

\subsection{Anticipatory Reading}

Not all reading behaviors, however, can be characterized as reactive.
As a concrete example, readers often skip words---a decision which must be made while the next word's identity is unknown. 
Furthermore, prior work  has shown that the uncertainty over a sentence’s continuations impacts RTs \cite[where this uncertainty is quatified as contextual entropy, as we make explicit later;][]{roark-etal-2009-deriving,angele2015successor,van2017approximations,van-schijndel-linzen-2019-entropy}. Both of these observations offer initial evidence that some form of \defn{anticipatory} planning is being performed by the reader, influencing the way that they read a text.\looseness=-1

The presence of such forms of anticipatory processing suggests that, beyond a word's surprisal, a reader's predictions about a word may influence the time they take to process it.
A word's RT, for instance, could be (at least partly) planned before arriving at it, based on the reader's expectation of the amount of work  necessary for processing that word. 
This expectation has a formal definition, the \defn{contextual entropy}, which is defined as follows:
\begin{subequations}\label{eq:shannon_ent}
\begin{align}
    \entfunc(&W_t \mid \boldsymbol{W}_{<t} = \prevwords) \defeq \expect_{\word \sim p(\cdot \mid \prevwords)}[h_t(\word)] \\
    &=-\! \sum_{\word \in \vocab} p(\word \mid \prevwords) \log_2 p (\word \mid \prevwords)
\end{align}
\end{subequations}
where $W_t$ denotes a $\vocab$-valued random variable, which takes on values $\word \in \vocab$ with distribution $p(\cdot \mid \prevwords)$.
Specifically, we assume a (potentially infinite) vocabulary $\vocabnoeos$, which we augment to include a special $\eos \not\in \vocabnoeos$ token that indicates the end of an utterance.
To that end, we define $\vocab \defeq \vocabnoeos \cup \{\eos\}$.
When clear from context, we shorten $\entfunc(W_t \mid \boldsymbol{W}_{<t} = \prevwords)$ to simply $\entfunc(W_t)$.\looseness=-1

Prior work has also investigated the role of entropy in RTs.
\citet{hale2003information,hale2006uncertainty}, for instance, investigated the role of entropy reduction on reading times; \citeauthor{hale2003information} defines entropy reduction as the change in the conditional entropy over sentence parses  
induced by word $t$, which is a different measure than the word entropy we investigate here.
More recently, other work \citep{roark-etal-2009-deriving,van2017approximations,aurnhammer2019evaluating} investigated the role of successor entropy (i.e., word $(t+1)$'s entropy) on RTs.
\citet{linzen-jaeger-2014-investigating} investigated both entropy reduction, total future entropy (i.e., the conditional entropy over sentence parses), and single step syntactic entropy (i.e.,  the entropy over the next step in a syntactic derivation).
In this work, we are instead interested in the role of the entropy of word $t$ itself because of its theoretical motivation as the expected processing difficulty under surprisal theory.

Similarly to us, \citet{cevoli2022prediction} also studies the role of the entropy of word $t$ on RTs for word $t$; more specifically, \citeauthor{cevoli2022prediction} analyze \textit{prediction error costs} by investigating how surprisal and entropy interact in predicting RTs.
Finally, \citet{smith-levy:2010cogsci} also investigate how word $t$'s contextual entropy influences RTs, but while further conditioning a reader's predictions on a noisy version of word $t$'s visual signal.\looseness=-1

\vspace{-2pt}
\subsection{Skewed Anticipations} \label{sec:skewed_anticipations}

\Cref{eq:shannon_ent} assumes readers will use the expectation to estimate what the surprisal of the anticipated word $\word_t$ will be, while knowing only its context $\prevwords$.
However, a reader may employ a different strategy when making anticipatory predictions.
One possibility, for instance, is that readers could be overly confident, and trust their best (i.e., most likely) guess when making this prediction.
In this case, readers would instead anticipate a subsequent word's surprisal to be:\looseness=-1%
\begin{subequations}\label{eq:argmin_ent}
\begin{align} 
     \entfunc_{\infty}& (\Word_t \mid \Words_{<t} = \prevwords) \defeq \min_{\word \in \vocab} h_t(\word) \\
     &= \min_{\word \in \vocab} - \log_2 p (\word \mid \prevwords)
\end{align}
\end{subequations}
where we use the notation $\mathrm{H}_{\infty}$ to describe this quantity for reasons that will become clear later in this section.
Another possibility is that readers could ignore each word's specific probability value when anticipating future surprisals, focusing instead on the number of competing possible words with non-zero probability:\footnote{While most state-of-the-art language models cannot assign zero probabilities to a word due to their use of a softmax in their final layers it is plausible that humans could.}\looseness=-1%
\begin{equation} \label{eq:argmax_ent}
    \entfunc_0(\Word_t \mid \Words_{<t} = \prevwords) \defeq - \log_2 \frac{1}{|\psupport|}
\end{equation}
where $\psupport \defeq \{\word \in \vocab \mid p(\word \mid \prevwords) > 0\}$.
As the subscript notation suggests, the above anticipatory predictions can be written in a unified framework using the \defn{contextual \renyi{} entropy} \citep{renyi1961measures}, which is defined as
\begin{align} \label{eq:renyi}
    \mathrm{H}_\alpha(\Word_t &\mid \boldsymbol{W}_{<t} = \prevwords) \\
    &\defeq \lim_{\beta \rightarrow \alpha} \frac{1}{1 - \beta} \log_2 \!\sum_{\word \in \vocab}\! \bigg( p(\word \mid \prevwords) \bigg)^{\beta} \nonumber
\end{align}
Again, we will shorten  $\mathrm{H}_\alpha(\Word_t \mid \boldsymbol{W}_{<t} = \prevwords)$  to 
$\mathrm{H}_\alpha(\Word_t)$.
Notably, with different values of $\alpha$, the \renyi{} entropy leads to different interpretations of a reader's anticipation strategies.
The \renyi{} entropy is equivalent to \cref{eq:argmax_ent} when $\alpha = 0$, which measures the size of the support of $p(\cdot\mid \prevwords)$, or equivalently, the number of competing (word-level) continuations at a timestep.
Further, the \renyi{} entropy is equivalent to \cref{eq:argmin_ent} when $\alpha = \infty$, which measures the surprisal of the word with maximal probability in a context.
Finally, through L'Hôpital's rule, it is equivalent to \cref{eq:shannon_ent}, the Shannon entropy, when $\alpha = 1$. 
In general, however, \renyi{} entropy does not have as clear of an intuitive meaning when $\alpha \notin \{0, 1, \infty\}$.
Notably, the \renyi{} entropy with $\alpha=\sfrac{1}{2}$ will be relevant in our results' section.
In this case, it can be thought of as measuring a softened version of a distribution $p$'s support.\looseness=-1

\section{Anticipatory Mechanisms} \label{sec:mechanisms}
In this paper, we are mainly interested in the effect of anticipations on RTs, where  we operationalize anticipation in terms of contextual (\renyi{}) entropy, as defined above.
We consider four main mechanisms under which anticipation could affect RTs: word skipping, budgeting, preemptive processing, uncertainty cost.
We discuss each of these in turn.\looseness=-1

\newcommand{\paragraphsvspace}{-1pt}

\paragraph{Word skipping.} 
The first way in which anticipation could affect RTs is by allowing readers to skip words entirely, allocating the word a reading time of zero.
A reader must, by definition, decide whether or not to skip a word \emph{before} fixating on it.
We hypothesize the reader may thus decide to skip a word when they are confident in its identity, i.e., when the word's contextual entropy is low.\footnote{This is under the assumption that the reader is not able to identify upcoming words through their parafoveal vision.} 
If this hypothesis is true, then contextual entropy should be a good predictor of when a reader skips words.\looseness=-1

\vspace{\paragraphsvspace}
\paragraph{Budgeting.} 
The reading process can be described as a sequence of fixations and saccades.\footnote{Fixations are when the gaze focuses on a word; saccades are rapid eye movements shifting gaze from a point to another. 
In self-paced reading, saccades are similar to mouse clicks.}
Saccades, however, do not happen instantly: On average, they must be planned at least 125 milliseconds in advance \citep{reichle2009using}.
Further, there is an average eye-to-brain delay of 50ms \citep{pollatsek2008}.
We may thus estimate that the effects of a word's surprisal, as well as other word properties such as frequency, in RT allocation will only show up 175ms after that word is fixated, or later.\footnote{Again, this assumes that the reader has not identified the word parafoveally. 
A second caveat regarding this analysis is that, once a saccade is initiated, there is an initial period during which it can be canceled or reprogrammed to target a different location \citep{vangisbergen1987stimulus}.}
Considering this delay on saccade execution, it is not unreasonable that RTs could be decided (or budgeted) further in advance, when the reader still does not know word $\word_t$'s identity.
If a reader indeed budgets reading times beforehand, RTs should be---at least in part---predictable from the contextual entropy.
Processing costs, however, may still be driven by surprisal. 
In this case, we might observe budgeting effects: e.g., if a reader \emph{under}-budgets RTs for a word, i.e., if the word's contextual entropy is smaller than its actual surprisal, we may see a compensation, which could manifest as larger spillover effects in the following word.\looseness=-1

\vspace{\paragraphsvspace}
\paragraph{Preemptive Processing.}
Recent work \citep[e.g.,][]{willems2015prediction,goldstein2022shared} suggests that---especially for low entropy contexts---the brain starts preemptively processing future words before reaching them.\footnote{Specifically, they show the brain's processing load before a word's onset correlates negatively with its entropy.} 
Thus, shorter reading times in low entropy words at time $t+1$ may be compensated by longer times in the previous word $\word_t$.
However, recent work investigating the effect of successor entropy, i.e., word $(t+1)$'s entropy, on RTs has found conflicting results.\footnote{While \citet{roark-etal-2009-deriving} and \citet{van2017approximations} have found that successor entropy has a positive impact on RTs, i.e., that when $\word_{t+1}$ has lower entropy, word $\word_{t}$ takes a shorter time to be read, both \citet{linzen-jaeger-2014-investigating} and \citet{aurnhammer2019evaluating} have found no effect.}
Also on the topic of preprocessing, \citet{smith2008optimal} derive from first principles what a reader's optimal preprocessing effort should be for any given context: 
Under their assumption that reading times should be scale-free and that readers optimally trade off preprocessing and reading costs, a reader should always allocate a \emph{constant} amount of resources for preprocessing future words.\footnote{We give the full derivation---including the necessary assumptions---in \cref{app:constant_preprocessing_time} for completeness.}
We will investigate the effect of successor entropy on RTs in \cref{sec:preemptive_processing}.\looseness=-1

\vspace{\paragraphsvspace}
\paragraph{Uncertainty Cost.}
Finally, uncertainty about a word's identity, as quantified by its contextual entropy, may cause an increase in processing load directly.
For example, keeping a large number of competing word continuations under consideration may require additional cognitive resources, impacting the reader's processing load beyond the effect of the observed word's surprisal.
We know, however, no way of testing this hypothesis directly under our experimental setup.
Therefore, we will not analyze this mechanism specifically; only studying it in our main experiment (\cref{sec:exp_entropy_vs_surprisal}), where it is measured in aggregate with other mechanisms.

\vspace{-3pt}
\section{Experimental Setup}

\vspace{-3pt}
\subsection{Estimators} \label{sec:estimating_info}
\vspace{-2pt}

Unfortunately, we cannot compute the values discussed in \cref{sec:predictive_rt}, as we do not have access to the true natural language distribution $p(\cdot \mid \prevwords)$.
We can, however, estimate these values using a language model $\ptheta(\cdot \mid \prevwords)$.
We will thus use $\ptheta$ in place of $p$ in order to estimate all the information-theoretic quantities in \cref{sec:predictive_rt}. 
Using language model-based estimators is standard practice when investigating the relationship between RTs and information-theoretic quantities, e.g., surprisal.

\vspace{-2pt}
\paragraph{Language Models.} 
We use \gptsmall{} \cite{gpt2} as our language model $\ptheta$ in all experiments.\footnote{We make use of \citeposs{wolf-etal-2020-transformers} library.}
Although some work has shown that a language model's quality correlates with its psychometric predictive power \citep{goodkind-bicknell-2018-predictive,wilcox2020predictive}, both \citet{shain2022large} and \citet{oh2022why} have more recently found that \gptsmall{}'s surprisal estimates are actually more predictive of RTs than those of both larger versions of \gpt{} and GPT-3.
We note, however, that \gpt predicts subwords at each time-step, rather than predicting full words.
Thus, to get word-level surprisal, we must sum over the subwords' surprisal estimates.
In some cases, many distinct subword sequences may represent a single word.
In this case, we only consider the \emph{canonical} subword sequence output by \gpt's tokenizer.
Estimating the contextual entropy per word is harder because computing it requires summing over the entire vocabulary $\vocab$, whose cardinality can be infinite.
We approximate the contextual entropy 
by computing the entropy over the subwords instead.\footnote{There are many ways to estimate the Rényi entropy, e.g., one could also have estimated the Rényi entropy by assuming a fixed finite vocabulary $\vocab$, and then computed the probability of the words' canonical tokenizations.\looseness=-1}
In practice, this is equivalent to computing a lower bound on the true contextual entropies, as we show in \cref{app:lowerbound}.\looseness=-1

\vspace{-5pt}
\subsection{Data}
\vspace{-2pt}

We perform our analyses on two eye-tracking and two self-paced reading datasets.
The self-paced reading corpora we study are the Natural Stories Corpus \citep{futrell-etal-2018-natural} and the Brown Corpus \citep{smith2013-log-reading-time}.
The eye-tracking corpora are the Provo Corpus \citep{provo} and the Dundee Corpus \citep{dundee}.
We refer readers to \cref{app:data} for more details on these corpora, as well as dataset statistics and preprocessing steps.
For the eye-tracking data, we focus our analyses on Progressive Gaze Duration: 
A word's RT is taken to be the sum of all fixations on it before a reader first passes it, i.e.,
we only consider fixations in a reader's first forward pass. 
Further, for our first set of experiments, we consider a skipped word’s RT to be zero {\citep[following][]{rayner2011eye};\footnote{This choice goes against the more common practice of simply discarding skipped words from the analyses. 
Our experimental paradigm is based on two factors.
First, we are interested in word skipping as a mechanism by which anticipation impacts RTs. 
Second, we want to make the eye-tracking setting more closely comparable to the self-paced reading, where fully skipping a word is not possible.\looseness=-1
}
we denote these datasets as Provo (\cmark) and Dundee (\cmark).
In later experiments, we discard skipped words, denoting these datasets with an (\xmark{}) instead.
Following prior work \citep[e.g.,][]{wilcox2020predictive}, we average RT measurements across readers, analyzing one RT value per word token.\looseness=-1  

\subsection{Linear Modeling}

Prior work has shown the surprisal--RT relationship to be mostly linear \citep{smith2008optimal,smith2013-log-reading-time,shain2022large}. 
Assuming this linearity extends to the contextual entropy--RT relationship,
we restrict our predictive function to be linear:\footnote{As both Shannon and \renyi{} entropies are linear functions of surprisal, we believe this assumption is justifiable.\looseness=-1}
$\predfunc(\bx) = \gamparams^{\intercal} \bx$, where $\gamparams$ is a column vector which parameterizes $\predfunc$.
Further, given data $\dataset = \{(\bx_n, \timefunc_n)\}_{n=1}^N$, regressor $\predfunc(\bx)$'s \emph{average} log-likelihood on $\dataset$ is given by\looseness=-1%
\begin{align}
&\llh(\predfunc(\bx)) = \frac{1}{N} \log \prod_{n=1}^N\frac{e^{-\frac{\left(\timefunc_n - \predfunc(\bx_n)\right)^2}{2\sigma^2}} }{\sqrt{2\pi \sigma^2}} \nonumber \\
&= - \frac{1}{N}\sum_{n=1}^N \left( \log \sqrt{2\pi \sigma^2} + \frac{\left(\timefunc_n - \predfunc(\bx_n)\right)^2}{2\sigma^2} \right)   \nonumber \\
&= - \log \sqrt{2\pi \sigma^2} - \sum_{n=1}^N \frac{\left(\timefunc_n - \predfunc(\bx_n)\right)^2 }{2N \sigma^2}
\end{align}
assuming Gaussian errors with variance $\sigma^2 > 0$.\footnote{We note that RTs cannot be negative, and thus prediction errors will not actually be Gaussian.}\looseness=-1

\subsection{Evaluation}

We evaluate the different sentence processing hypotheses by looking at the predictive power of their associated regressors.
Predictive power is quantified as the log-likelihood on held-out data.
We use 10-fold cross-validation, estimating our regressors, given in \cref{eq:regressor}, using 9 folds of the data at a time, and evaluating them on the 10$^\text{th}$ fold.
Further, as is standard in RT analyses, we test the predictive power of a hypothesis by comparing a target model against a baseline model.
These models differ only in that the target model contains a predictor of interest, whereas the baseline model does not.
Our metric of interest is thus the difference in log-likelihood of held-out data between the two models:%
\footnote{\label{footnote:significance}Significance is assessed using a paired permutation test.
We correct for multiple hypothesis testing \citep{benjamini1995controlling} and mark: \textcolor{mygreen}{in green} significant $\deltallh$ where a variable adds predictive power (i.e., when the model with more predictors is better), \textcolor{myred}{in red} significant $\deltallh$ where a variable leads to overfitting (i.e., when the model with more predictors is worse).
$^*$~$p<0.05$, $^{**}$~$p<0.01$, $^{***}$~$p<0.001$.\looseness=-1}\looseness=-1%
\begin{equation}
    \deltallh = \llh(\predfunc(\bxmodel)) - \llh(\predfunc(\bxbase))
\end{equation}
which, when positive, indicates the target model explains this data better than the base model.\looseness=-1

\begin{table}[t]
    \centering
\resizebox{\columnwidth}{!}{%
    \begin{tabular}{lccccc}
    \toprule
         &
         & \multicolumn{4}{c}{Surprisal}
         \\
         \cmidrule(lr){3-6}
         &
         & $\word_{t\!-\!3}$ & $\word_{t\!-\!2}$ & $\word_{t\!-\!1}$ & $\word_{t}$
         \\
\midrule

\multicolumn{2}{l}{Brown} & \phantom{-}\textcolor{mygreen}{0.33}$^{***}$& \phantom{-}\textcolor{mygreen}{0.47}$^{***}$& \phantom{-}\textcolor{mygreen}{2.58}$^{***}$& \phantom{-}\textcolor{mygreen}{0.50}$^{*}$\phantom{$^{**}$} \\
\multicolumn{2}{l}{Natural Stories} & \phantom{-}\textcolor{mygreen}{0.20}$^{*}$\phantom{$^{**}$}& \phantom{-}\textcolor{mygreen}{0.34}$^{*}$\phantom{$^{**}$}& \phantom{-}\textcolor{mygreen}{1.05}$^{***}$& \phantom{-}\textcolor{mygreen}{1.54}$^{***}$ \\
Provo & (\cmark) & \phantom{-}{0.07}\phantom{$^{***}$}& \phantom{-}{0.18}\phantom{$^{***}$}& \phantom{-}\textcolor{mygreen}{0.83}$^{*}$\phantom{$^{**}$}& \phantom{-}\textcolor{mygreen}{3.22}$^{**}$\phantom{$^{*}$} \\
Dundee & (\cmark) & {-0.00}\phantom{$^{***}$}& \phantom{-}\textcolor{mygreen}{0.04}$^{**}$\phantom{$^{*}$}& \phantom{-}\textcolor{mygreen}{0.25}$^{***}$& \phantom{-}\textcolor{mygreen}{0.89}$^{***}$ \\

\bottomrule 
\\[-8pt]
\multicolumn{6}{l}{\small 
$\bxmodel_t = \bxcommon_t \oplus \bxsurp_t$\,\,\,\,\,\, vs.  \,\,\,\,\,\,
$\bxbase_t = \bxcommon_t \oplus \bxsurpnoti_t$
}
    \end{tabular}
}
\vspace{-5pt}
    \caption{$\deltallh$ (in $10^{-2}$ nats) when comparing a model with all surprisal terms against baselines from which a single surprisal term was removed. Green indicates a significantly positive impact of surprisal on the model's predictive power.}
    \label{tab:delta_llh_surprisal}
    \vspace{-5pt}
\end{table}

\section{Experiments and Results}

\subsection{Experiment \#1: Confirmatory Analysis}\label{sec:surprisal_analysis}
In the first experiment, we confirm prior results that show the predictive power of surprisal on RTs.
First, we define the following sets of predictors:
\begin{subequations}
\begin{align}
    &\bxcommon_t = [|\word_t|, u(\word_t), \dots, |\word_{t\!-\!3}|, u(\word_{t\!-\!3})]^{\intercal} \\
    &\bxsurp_t = [h_{t}(\word_{t}), \dots, h_{t\!-\!3}(\word_{t\!-\!3})]^{\intercal} \\
    &\bxsurpnott_t = [h_{t\!-\!1}(\word_{t\!-\!1}), \dots, h_{t\!-\!3}(\word_{t\!-\!3})]^{\intercal}
\end{align}
\end{subequations}
where $|w_t|$ is the word length in characters and $u(w_t)$ is the unigram frequency of the $t^{\text{th}}$ word.
Notably, we include predictors for words $\word_{t\!-\!1}$, $\word_{t\!-\!2}$, and $\word_{t\!-\!3}$ because prior work has shown that a word's RT is impacted not only by its own surprisal, but also by the surprisal of previous words.
These effects are referred to as \defn{spillover effects}.
We then estimate the $\deltallh$ between
\begin{align}
    &\bxmodel_t = \bxcommon_t \oplus \bxsurp_t \\
    &\bxbase_t = \bxcommon_t \oplus \bxsurpnoti_t
\end{align}
where $\oplus$ stands for the vertical concatenation of two vectors and $t' \in \{t, t\!-\!1, t\!-\!2, t\!-\!3\}$.
In words, $\bxmodel$ includes all surprisal predictors $\bxsurp_t$, while for the baseline model $\bxbase$ we remove surprisal predictors one at a time.}
We present these results in \cref{tab:delta_llh_surprisal}.
The results show that the surprisal of word $\word_t$ is a strong predictor of RTs in all four analyzed datasets.
Additionally, we see significant spillover effects for the surprisal of three previous words in self-paced reading corpora, for the two previous words in Dundee, and for the single previous one in Provo.
Interestingly, and consistent with prior work \citep{smith2008optimal,smith2013-log-reading-time}, we find that spillover effects are stronger than the current word's effect in Brown.
On the other three datasets, however, we find the surprisal effect on the current word to be stronger than the spillover effects.

\begin{table}[t]
    \centering
\resizebox{\columnwidth}{!}{%
    \begin{tabular}{lccccc}
    \toprule
         && $\word_{t\!-\!3}$ & $\word_{t\!-\!2}$ & $\word_{t\!-\!1}$ & $\word_{t}$
         \\
\midrule
\multicolumn{6}{l}{Replace Surprisal with Entropy$^1$} \\
\cmidrule(lr){1-1}
    \multicolumn{2}{l}{Brown} & \textcolor{myred}{-0.30}$^{*}$\phantom{$^{**}$}& \textcolor{myred}{-0.35}$^{**}$\phantom{$^{*}$}& \textcolor{myred}{-1.68}$^{***}$& {-0.03}\phantom{$^{***}$} \\
    \multicolumn{2}{l}{Natural Stories} & {-0.03}\phantom{$^{***}$}& \textcolor{myred}{-0.19}$^{*}$\phantom{$^{**}$}& \textcolor{myred}{-0.41}$^{*}$\phantom{$^{**}$}& \phantom{-}{0.37}\phantom{$^{***}$} \\
    Provo & (\cmark) & {-0.08}\phantom{$^{***}$}& \phantom{-}{0.18}\phantom{$^{***}$}& \textcolor{myred}{-0.66}$^{*}$\phantom{$^{**}$}& \textcolor{myred}{-2.58}$^{*}$\phantom{$^{**}$} \\
    Dundee & (\cmark) & {-0.00}\phantom{$^{***}$}& \phantom{-}\textcolor{mygreen}{0.03}$^{*}$\phantom{$^{**}$}& \textcolor{myred}{-0.21}$^{***}$& {-0.07}\phantom{$^{***}$} \\
\midrule
\multicolumn{6}{l}{Add Entropy$^2$} \\
\cmidrule(lr){1-1}
    \multicolumn{2}{l}{Brown} & \textcolor{myred}{-0.03}$^{*}$\phantom{$^{**}$}& {-0.01}\phantom{$^{***}$}& \phantom{-}{0.04}\phantom{$^{***}$}& \phantom{-}\textcolor{mygreen}{0.15}$^{*}$\phantom{$^{**}$} \\
    \multicolumn{2}{l}{Natural Stories} & \phantom{-}{0.04}\phantom{$^{***}$}& \phantom{-}{0.01}\phantom{$^{***}$}& \phantom{-}\textcolor{mygreen}{0.14}$^{***}$& \phantom{-}\textcolor{mygreen}{0.89}$^{***}$ \\
    Provo & (\cmark) & \textcolor{myred}{-0.04}$^{*}$\phantom{$^{**}$}& \phantom{-}{0.16}\phantom{$^{***}$}& {-0.03}\phantom{$^{***}$}& {-0.06}\phantom{$^{***}$} \\
    Dundee & (\cmark) & \textcolor{myred}{-0.00}$^{**}$\phantom{$^{*}$}& \phantom{-}\textcolor{mygreen}{0.03}$^{**}$\phantom{$^{*}$}& {-0.00}\phantom{$^{***}$}& \phantom{-}\textcolor{mygreen}{0.25}$^{***}$ \\
\bottomrule
\\[-8pt]
\multicolumn{6}{l}{\small 
$^1$ 
$\bxmodel_t = \bxcommon_t \oplus \bxsurpnoti_t \oplus [\enti]^\intercal$
} \\\multicolumn{6}{l}{\small 
$^2$
$\bxmodel_t = \bxcommon_t \oplus \bxsurp_t \oplus [\enti]^\intercal$
} \\\multicolumn{6}{l}{\small 
$^\mathrm{both}$
$\bxbase_t = \bxcommon_t \oplus \bxsurp_t$,
}
    \end{tabular}
}
\vspace{-5pt}
    \caption{$\deltallh$ (in $10^{-2}$ nats) achieved after either replacing a surprisal term in the baseline with Shannon's entropy (top), or adding the entropy as an extra predictor (bottom). Green indicates a significant gain in $\deltallh$, red a significant loss.}
    \label{tab:delta_llh_entropy}
\end{table}

\subsection{Experiment \#2: Surprisal vs. Entropy} \label{sec:exp_entropy_vs_surprisal}

In the second experiment, we analyze the predictive power of the contextual Shannon entropy on RTs.
Specifically, \cref{tab:delta_llh_entropy} presents the $\deltallh$ between the baseline model 
$\bxbase_t = \bxcommon_t \oplus \bxsurp_t$ and two target models.
The first is a model where the entropy term $\ent$ is added \emph{in addition} to the predictors already present in $\bxbase$.
The second is a model where the surprisal term $h_t(w_t)$ is replaced by the entropy term $\ent$.
From \cref{tab:delta_llh_entropy}, we see that adding the entropy of the current word significantly increases the predictive power in three out of the four analyzed datasets.
Furthermore, replacing the surprisal predictor with the entropy only leads to a model with worse predictive power in one of the three analyzed datasets (in Provo).
On the other three datasets, the entropy's predictive power is as good as the surprisal's---more precisely, they are not statistically different.
Together, these results suggest that the reading process is both responsive and anticipatory.\looseness=-1

Analyzing the impact of the previous words' entropies, i.e., $\entprev$,  $\entprevprev$,  $\entprevprevprev$,\footnote{We term these predictors \emph{spillover} entropy effects by analogy to the surprisal case. As before, we omit the conditioning factor on these entropies for notational succinctness, i.e., we write $\entprev$ instead of $ \entfunc(\Word_{t-1} \mid \boldsymbol{W}_{<t-1} = \words_{< t-1})$.} on RTs, we see a somewhat different story. 
When adding spillover entropy terms as extra predictors we see no consistent improvements in predictive power.
We observe a weak improvement on self-paced reading datasets when adding $\entprev$ as a predictor, but, even then, the improvement is only significant on Natural Stories.
We find a similarly weak effect when adding $\entprevprev$ on eye-tracking data, which is only significant on the Dundee corpus. 
This lack of predictive power further stands out when contrasted to surprisal spillover effects, which were mostly significant; see \cref{tab:delta_llh_surprisal}.
Furthermore, replacing surprisal spillover terms with the corresponding entropy terms generally leads to models with weaker predictive power.
Together, these results imply the effect of entropy (expected surprisal) on RTs is mostly local, i.e., the expectation over a word's surprisal impacts its RT, but not future words' RTs.\looseness=-1

\begin{figure*}[t]
    \centering
    \begin{subfigure}[b]{.25\textwidth}
        \includegraphics[width=\columnwidth,trim={0 \trimheight{} 0 0},clip]{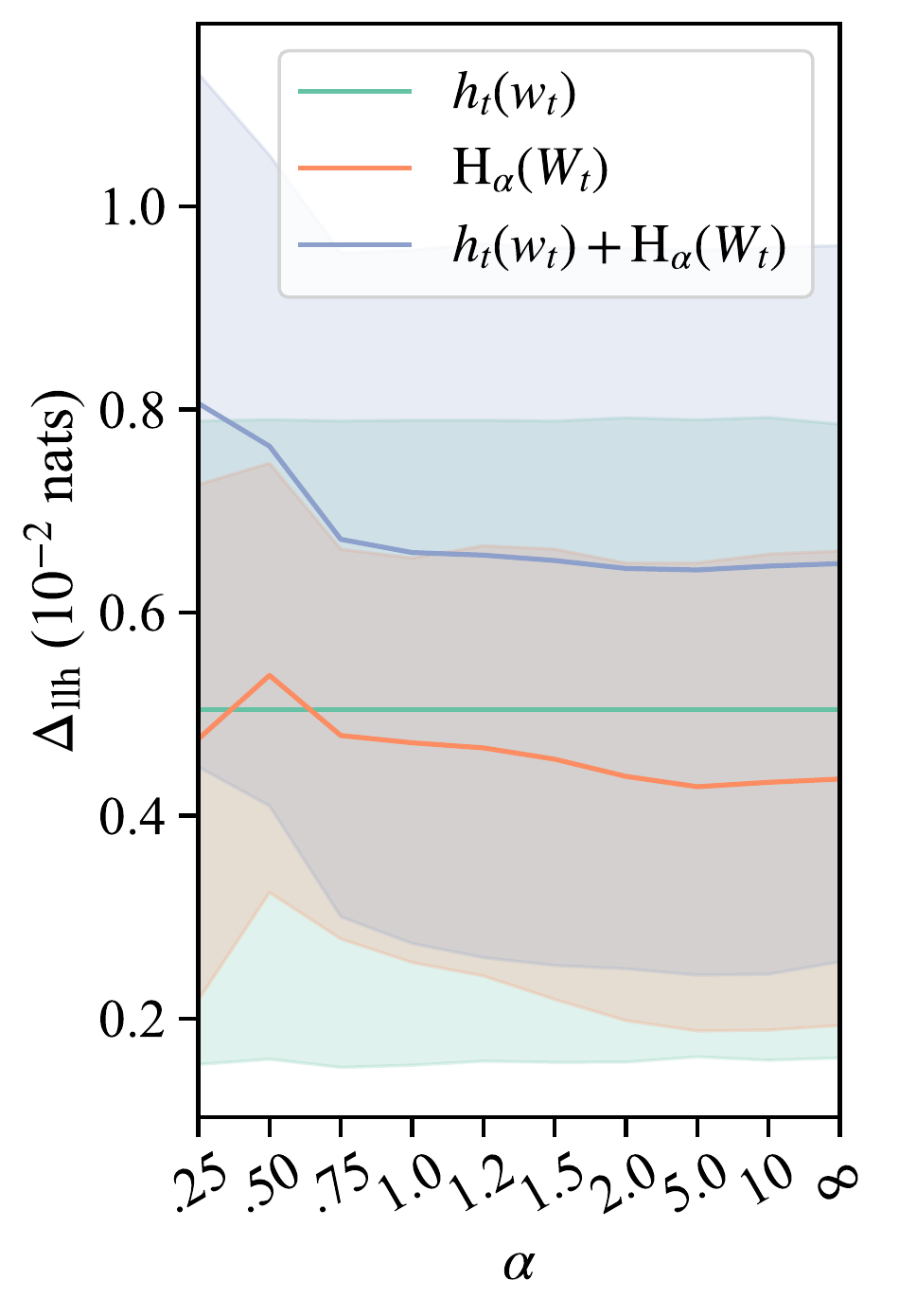}
        \caption{Brown}
    \end{subfigure}%
    ~
    \begin{subfigure}[b]{.25\textwidth}
        \includegraphics[width=\columnwidth,trim={0 \trimheight{} 0 0},clip]{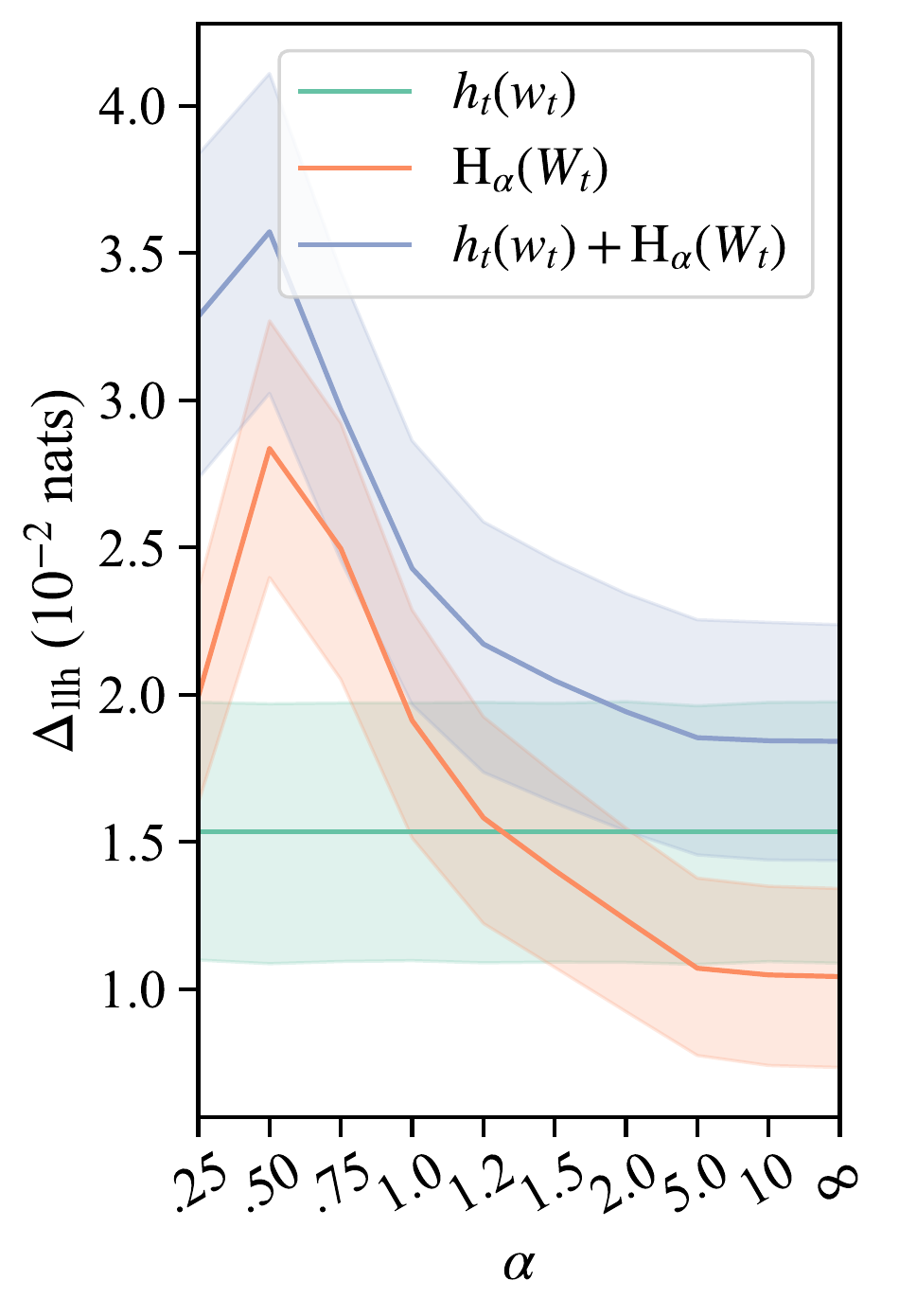}
        \caption{Natural Stories}
    \end{subfigure}%
    ~
    \begin{subfigure}[b]{.25\textwidth}
        \includegraphics[width=\columnwidth,trim={0 \trimheight{} 0 0},clip]{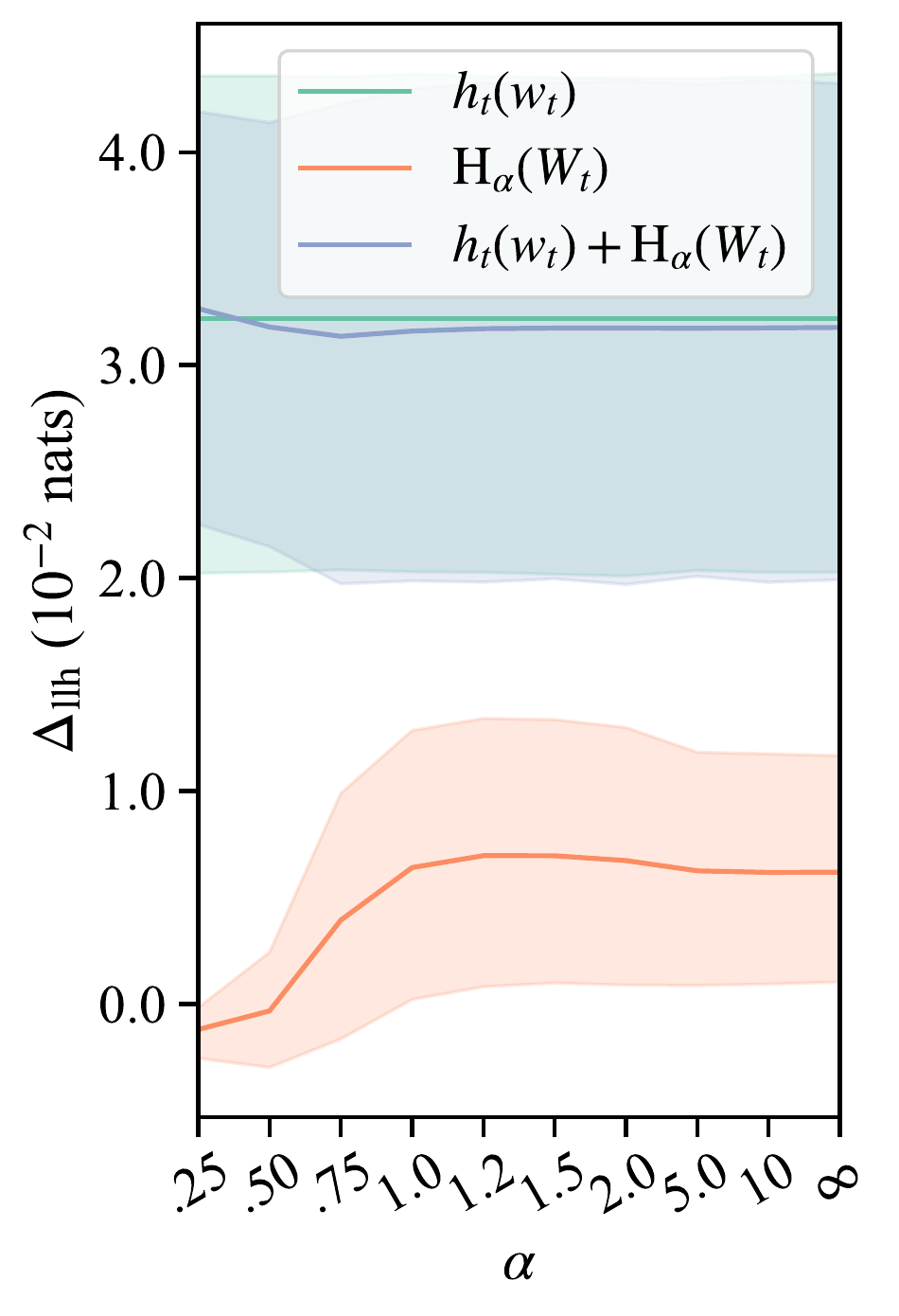}
        \caption{Provo (\cmark)}
    \end{subfigure}%
    ~
    \begin{subfigure}[b]{.25\textwidth}
        \includegraphics[width=\columnwidth,trim={0 \trimheight{} 0 0},clip]{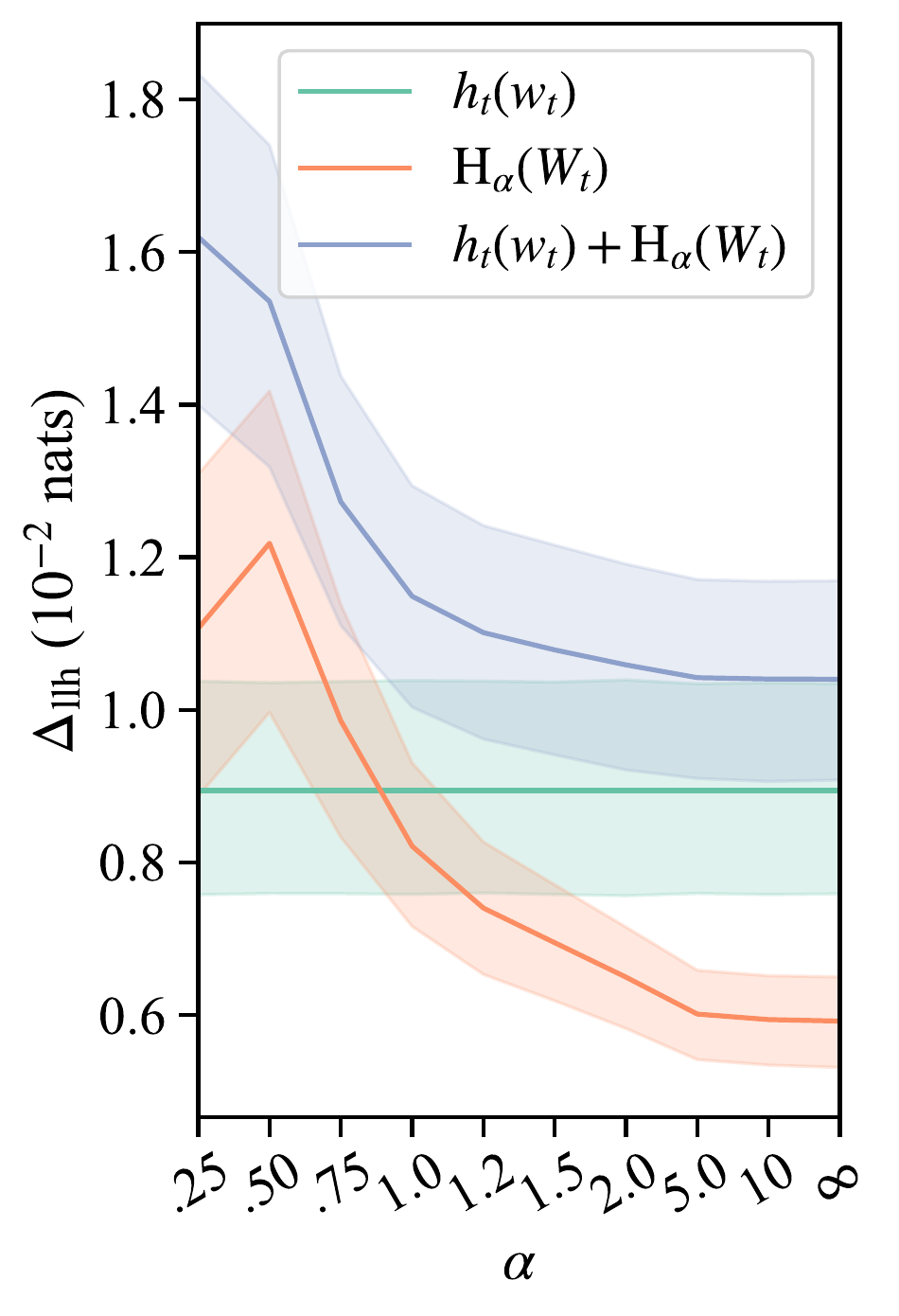}
        \caption{Dundee (\cmark)}
    \end{subfigure}
    \vspace{-17pt}
    \caption{$\deltallh$ when adding either the current word's surprisal, \renyi{} entropy, or both on top of a baseline that includes the surprisal of previous words as predictors, i.e., $\bxbase = \bxcommon_t \oplus \bxsurpnott$.
    Shaded regions correspond to 95\% confidence intervals.}
    \label{fig:renyi}
    \vspace{-5pt}
\end{figure*}

\begin{table*}
    \centering
\resizebox{.85\textwidth}{!}{%
    \begin{tabular}{lccccccccc}
    \toprule
         &
         & \multicolumn{4}{c}{Replace Surprisal with \renyi{} Entropy$^1$} 
         & \multicolumn{4}{c}{Add \renyi{} Entropy$^2$}
         \\
         \cmidrule(lr){3-6}
         \cmidrule(lr){7-10}
         &
         & $\word_{t\!-\!3}$ & $\word_{t\!-\!2}$ & $\word_{t\!-\!1}$ & $\word_{t}$
         & $\word_{t\!-\!3}$ & $\word_{t\!-\!2}$ & $\word_{t\!-\!1}$ & $\word_{t}$
         \\
\midrule

\multicolumn{2}{l}{Brown} & \textcolor{myred}{-0.35}$^{***}$& \textcolor{myred}{-0.37}$^{**}$\phantom{$^{*}$}& \textcolor{myred}{-1.76}$^{***}$& \phantom{-}{0.03}\phantom{$^{***}$}& {-0.01}\phantom{$^{***}$}& \phantom{-}{0.00}\phantom{$^{***}$}& \phantom{-}{0.14}\phantom{$^{***}$}& \phantom{-}\textcolor{mygreen}{0.26}$^{*}$\phantom{$^{**}$} \\
\multicolumn{2}{l}{Natural Stories\,\,\,\,\,\,} & {-0.16}\phantom{$^{***}$}& \textcolor{myred}{-0.27}$^{*}$\phantom{$^{**}$}& {-0.19}\phantom{$^{***}$}& \phantom{-}\textcolor{mygreen}{1.30}$^{**}$\phantom{$^{*}$}& {-0.00}\phantom{$^{***}$}& {-0.00}\phantom{$^{***}$}& \phantom{-}\textcolor{mygreen}{0.44}$^{***}$& \phantom{-}\textcolor{mygreen}{2.04}$^{***}$ \\
Provo & (\cmark) & {-0.05}\phantom{$^{***}$}& \phantom{-}{0.47}\phantom{$^{***}$}& \textcolor{myred}{-0.89}$^{*}$\phantom{$^{**}$}& \textcolor{myred}{-3.25}$^{**}$\phantom{$^{*}$}& {-0.01}\phantom{$^{***}$}& \phantom{-}\textcolor{mygreen}{0.45}$^{*}$\phantom{$^{**}$}& {-0.01}\phantom{$^{***}$}& {-0.04}\phantom{$^{***}$} \\
Dundee & (\cmark) & \phantom{-}{0.00}\phantom{$^{***}$}& \phantom{-}\textcolor{mygreen}{0.07}$^{*}$\phantom{$^{**}$}& \textcolor{myred}{-0.25}$^{***}$& \phantom{-}\textcolor{mygreen}{0.32}$^{*}$\phantom{$^{**}$}& {-0.00}\phantom{$^{***}$}& \phantom{-}\textcolor{mygreen}{0.08}$^{*}$\phantom{$^{**}$}& \phantom{-}{0.05}\phantom{$^{***}$}& \phantom{-}\textcolor{mygreen}{0.64}$^{***}$ \\

\bottomrule
\\[-8pt]
\multicolumn{10}{l}{\small 
$^1$ 
$\bxmodel_t = \bxcommon_t \oplus \bxsurpnoti_t \oplus [\renyienti]^\intercal$,
\,\,\,
$^2$
$\bxmodel_t = \bxcommon_t \oplus \bxsurp_t \oplus [\renyienti]^\intercal$,
\,\,\,
$^\mathrm{both}$
$\bxbase_t=\bxcommon_t \oplus \bxsurp_t$
}
    \end{tabular}
}
\vspace{-5pt}
    \caption{$\deltallh$ (in $10^{-2}$ nats) achieved after either replacing a surprisal term in the baseline with the contextual \renyi{} entropy ($\alpha=\sfrac{1}{2}$), or adding the \renyi{} entropy as an extra predictor.}
    \label{tab:delta_llh_entropy_renyi}
    \vspace{-5pt}
\end{table*}

\subsection{Experiment \#3: Skewed Expectations}
\label{sec:skewed_expectations_analysis}
We now compare the effect of \renyi{} entropy with $\alpha \neq 1$ on RTs.
We follow a similar setup to before.
Specifically, we compute the contextual \renyi{} entropy for several values of $\alpha$.
We then train regressors where we either add the \renyi{} entropy as an additional predictor, or where we replace the current word's surprisal $h_t(w_t)$ with the \renyi{} entropy. 
We then plot these values in \cref{fig:renyi}.
Analyzing this figure, we see that Provo again presents different trends from the other datasets.
We also see a clear trend in the three other datasets:
The predictive power of expectations seem to improve for smaller values of $\alpha$.
More precisely, in Brown, Natural Stories and Dundee, $\alpha=\sfrac{1}{2}$ seems to lead to stronger predictive powers than $\alpha>\sfrac{1}{2}$.\looseness=-1

Based on these results, we then produce a similar table to the previous experiment's, but  using the \renyi{} entropy with $\alpha=\sfrac{1}{2}$ instead.
These results are depicted in \cref{tab:delta_llh_entropy_renyi}.
Similarly to before, we still see a significant improvement in predictive powers on three of the datasets when adding the entropy as an extra predictor.
Unlike before, however, replacing the surprisal predictors (for time step $t$) with \renyi{} entropy predictors significantly improves log-likelihoods in two of the analyzed datasets.
In other words, the \renyi{} entropy has a stronger predictive power than the surprisal in both these datasets.
We now move on to investigate why this is the case, analyzing the mechanisms proposed in \cref{sec:mechanisms}.\looseness=-1

\begin{table*}[t]
    \centering
    \small
    \begin{tabular}{llcccccc}
    \toprule
& & \multicolumn{3}{c}{Provo} 
& \multicolumn{3}{c}{Dundee} 
\\
\cmidrule(lr){3-5}
\cmidrule(lr){6-8}
& 
& $h_t(w_t)$
& $\renyient$
& Both
& $h_t(w_t)$
& $\renyient$
& Both
\\
\midrule
\multirow{3}{*}{\shortstack[c]{Shannon \\ ($\alpha=1$)}}

& $\emptyset$ 
& \phantom{-}{2.60}
& \phantom{-}{1.76}
& \phantom{-}{2.86}

& \phantom{-}\textcolor{mygreen}{1.30}$^{*}$
& \phantom{-}\textcolor{mygreen}{2.50}$^{***}$ 
& \phantom{-}\textcolor{mygreen}{2.62}$^{***}$ 
\\

& $\surp$ 
& - & {-0.84}& \phantom{-}{0.26} 
& - & \phantom{-}{1.20}\phantom{$^{***}$}& \phantom{-}\textcolor{mygreen}{1.32}$^{***}$ 
\\

& $\renyient$ 
& - & - & \phantom{-}{1.10} 
& - & - & \phantom{-}{0.12}\phantom{$^{***}$}
\\

\midrule
\multirow{3}{*}{\shortstack[c]{\renyi{} \\ ($\alpha=\sfrac{1}{2}$)}}
& $\emptyset$ 
& \phantom{-}{2.60}
& \phantom{-}{0.84}
& \phantom{-}{2.66}

& \phantom{-}\textcolor{mygreen}{1.30}$^{*}$
& \phantom{-}\textcolor{mygreen}{5.10}$^{***}$
& \phantom{-}\textcolor{mygreen}{5.14}$^{***}$
\\
& $\surp$ 
& -
& {-1.76}
& \phantom{-}{0.06}

& -
& \phantom{-}\textcolor{mygreen}{3.79}$^{***}$
& \phantom{-}\textcolor{mygreen}{3.83}$^{***}$
\\

& $\renyient$ 
& -
& -
& \phantom{-}{1.82} 

& -
& -
& \phantom{-}{0.04}\phantom{$^{***}$}
\\
\bottomrule
    \end{tabular}
\vspace{-4pt}
    \caption{$\deltallh$ (in $10^{-4}$ nats) between a target model (with predictors on columns) vs baseline (with predictors on row) when predicting whether a word was skipped or not. 
    All models also include the surprisal of the previous words as predictors as well as length and unigram frequencies}.
    \label{tab:delta_llh_skip}
\vspace{-4pt}
\end{table*}

\subsection{Experiment \#4: Word Skipping} \label{sec:word-skip_planning}
In \cref{sec:mechanisms}, we discussed four potential mechanisms through which expectations could impact RTs.
In this experiment, we analyze the impact of word-skipping effects on our results.
We thus only consider the two eye-tracking datasets in this experiment, as self-paced reading does not allow for word skipping.
We start this analysis by, similarly to previous experiments, looking at the $\deltallh$ between a baseline model and an additional model that captures our target effect.
In contrast to previous experiments, though, we employ a logistic regressor that predicts whether or not a word was skipped during the readers' initial pass.
Our prediction function can thus be written as
\begin{equation}
    \predfunc(\bx) = \sigma(\gamparams^\intercal\, \bx)
\end{equation}
where $\gamparams$ is a column vector of the model's parameters and $\sigma$ is the sigmoid function.
Now, given data $\dataset = \{(\bx_n, \skipratio_n)\}_{n=1}^N$, where $\skipratio$ represents the ratio of readers who skipped a word,
the average log-likelihood of this predictor on $\dataset$ is:\looseness=-1%
\begin{align}
    &\llh(\predfunc(\bx)) = \\ 
    &\,\,\,\,\sum_{n=1}^N \frac{\skipratio_n \log \predfunc(\bx_n) + (1\!-\!\skipratio_n) \log (1\!-\!\predfunc(\bx_n)) }{N} \nonumber
\end{align}
Notably, having $y$ represent the ratio of readers who skipped a word---as opposed to the per-reader binary skipped vs not distinction---is equivalent to averaging the predicted feature across readers, as we do when predicting reading times.

\cref{tab:delta_llh_skip} presents our results.
First, we see that surprisal is a significant predictor of whether or not a word is skipped in Dundee; however, it is not a significant predictor in Provo.
Second, we find that in Dundee the  predictive power over whether a word was skipped is significantly stronger when using the \renyi{} entropy of the current word than when using its surprisal.
Finally, while we find an improvement in predictive power when adding entropy (in addition to surprisal) as a predictor, we find no significant improvement when starting with entropy and adding surprisal.
This implies that, at least for Dundee, word-skipping effects are predicted solely by the entropy, with the surprisal of the current word adding no extra predictive power.

\begin{table}
    \centering
\resizebox{\columnwidth}{!}{%
    \begin{tabular}{llcccc}
    \toprule
         & & $\word_{t\!-\!3}$ & $\word_{t\!-\!2}$ & $\word_{t\!-\!1}$ & $\word_{t}$
         \\
\midrule
\multicolumn{5}{l}{Shannon Entropy ($\alpha=1$)} \\
\cmidrule(lr){1-1}
\multirow{2}{*}{Replace$^1$}
& Provo & \phantom{-}{0.02}\phantom{$^{***}$}& {-0.03}\phantom{$^{***}$}& {-0.18}\phantom{$^{***}$}& \textcolor{myred}{-2.23}$^{***}$ \\
& Dundee & {-0.01}\phantom{$^{***}$}& {-0.02}\phantom{$^{***}$}& \textcolor{myred}{-0.15}$^{***}$& \textcolor{myred}{-0.32}$^{**}$\phantom{$^{*}$} \\
[7pt]
\multirow{2}{*}{Add$^2$}
& Provo & {-0.02}\phantom{$^{***}$}& {-0.07}\phantom{$^{***}$}& \phantom{-}{0.03}\phantom{$^{***}$}& {-0.01}\phantom{$^{***}$} \\
& Dundee & \textcolor{myred}{-0.00}$^{*}$\phantom{$^{**}$}& \phantom{-}{0.01}\phantom{$^{***}$}& \phantom{-}{0.01}\phantom{$^{***}$}& \phantom{-}\textcolor{mygreen}{0.17}$^{**}$\phantom{$^{*}$} \\
\midrule
\multicolumn{5}{l}{Renyi Entropy ($\alpha=\sfrac{1}{2}$)} \\
\cmidrule(lr){1-1}
\multirow{2}{*}{Replace$^1$}
& Provo & \phantom{-}{0.02}\phantom{$^{***}$}& \phantom{-}{0.10}\phantom{$^{***}$}& {-0.27}\phantom{$^{***}$}& \textcolor{myred}{-2.43}$^{**}$\phantom{$^{*}$} \\
& Dundee & {-0.01}\phantom{$^{***}$}& \phantom{-}{0.01}\phantom{$^{***}$}& \textcolor{myred}{-0.19}$^{***}$& {-0.18}\phantom{$^{***}$} \\
[7pt]
\multirow{2}{*}{Add$^2$}
& Provo & {-0.02}\phantom{$^{***}$}& \phantom{-}{0.07}\phantom{$^{***}$}& \phantom{-}{0.01}\phantom{$^{***}$}& \phantom{-}{0.32}\phantom{$^{***}$} \\
& Dundee & \textcolor{myred}{-0.00}$^{*}$\phantom{$^{**}$}& \phantom{-}\textcolor{mygreen}{0.05}$^{*}$\phantom{$^{**}$}& \phantom{-}{0.01}\phantom{$^{***}$}& \phantom{-}\textcolor{mygreen}{0.36}$^{***}$ \\
\bottomrule
\\[-8pt]
\multicolumn{6}{l}{
$^1$
$\bxmodel_t = \bxcommon_t \oplus \bxsurpnoti_t \oplus [\renyienti]^\intercal$, 
} \\ \multicolumn{6}{l}{
$^2$
$\bxmodel_t = \bxcommon_t \oplus \bxsurp_t \oplus [\renyienti]^\intercal$,
} \\ \multicolumn{6}{l}{
$^{\mathrm{both}}$
$\bxbase_t = \bxcommon_t \oplus \bxsurp_t$
}
    \end{tabular}
}
\vspace{-5pt}
    \caption{$\deltallh$ (in $10^{-2}$ nats) when predicting RTs on eye-tracking datasets where skipped words were removed, i.e., Provo (\xmark) and Dundee (\xmark).\looseness=-1}
    \label{tab:delta_llh_entropy_no_skipped}
\vspace{-4pt}
\end{table}

Note that we represented skipped words as having RTs of 0ms in our previous experiments on eye-tracking datasets. 
Thus, our previous results could be driven purely by word-skipping effects. 
We now run the same experiments as in \cref{sec:exp_entropy_vs_surprisal} and \cref{sec:skewed_expectations_analysis}, but with skipped words removed from our analysis.
These results are presented in \cref{tab:delta_llh_entropy_no_skipped}.
In short, when skipped words are not considered, the \renyi{} entropy is no more predictive of RTs than the surprisal.
In fact, the surprisal seems to be a slightly stronger predictor, albeit not significantly so in Dundee.
However, adding the \renyi{} entropy as a predictor to a model which already has surprisal still adds significant predictive power in Dundee.
In short, this table shows that, while partly driven by word skipping, there are still potentially other effects of anticipation on RTs.\looseness=-1

\begin{table*}
    \centering
\resizebox{\textwidth}{!}{%
    \begin{tabular}{lccccccccccccc}
\toprule
         &
         & \multicolumn{3}{c}{$\Delta$-budget} 
         & \multicolumn{3}{c}{Over-budget} 
         & \multicolumn{3}{c}{Under-budget} 
         & \multicolumn{3}{c}{$|\cdot|$-budget} 
         \\
         \cmidrule(lr){3-5}
         \cmidrule(lr){6-8}
         \cmidrule(lr){9-11}
         \cmidrule(lr){12-14}
         &
         & $\word_{t\!-\!3}$ & $\word_{t\!-\!2}$ & $\word_{t\!-\!1}$
         & $\word_{t\!-\!3}$ & $\word_{t\!-\!2}$ & $\word_{t\!-\!1}$
         & $\word_{t\!-\!3}$ & $\word_{t\!-\!2}$ & $\word_{t\!-\!1}$
         & $\word_{t\!-\!3}$ & $\word_{t\!-\!2}$ & $\word_{t\!-\!1}$
         \\
\midrule
\multicolumn{5}{l}{Shannon Entropy ($\alpha=1$)} \\
\cmidrule(lr){1-1}

\multicolumn{2}{l}{Brown} & {-0.03}\phantom{$^{***}$}& {-0.01}\phantom{$^{***}$}& \phantom{-}{0.02}\phantom{$^{***}$}& {-0.01}\phantom{$^{***}$}& {-0.01}\phantom{$^{***}$}& \phantom{-}{0.07}\phantom{$^{***}$}& {-0.02}\phantom{$^{***}$}& \textcolor{myred}{-0.01}$^{***}$& \textcolor{myred}{-0.02}$^{**}$\phantom{$^{*}$}& \phantom{-}{0.01}\phantom{$^{***}$}& {-0.01}\phantom{$^{***}$}& \phantom{-}{0.03}\phantom{$^{***}$} \\
\multicolumn{2}{l}{Natural Stories} & \phantom{-}{0.04}\phantom{$^{***}$}& \phantom{-}{0.00}\phantom{$^{***}$}& \phantom{-}{0.05}\phantom{$^{***}$}& {-0.01}\phantom{$^{***}$}& {-0.00}\phantom{$^{***}$}& \phantom{-}{0.02}\phantom{$^{***}$}& \phantom{-}{0.04}\phantom{$^{***}$}& {-0.00}\phantom{$^{***}$}& \phantom{-}{0.02}\phantom{$^{***}$}& {-0.01}\phantom{$^{***}$}& {-0.01}\phantom{$^{***}$}& \textcolor{myred}{-0.02}$^{*}$\phantom{$^{**}$} \\

Provo & (\cmark) & \textcolor{myred}{-0.04}$^{***}$& \phantom{-}{0.16}\phantom{$^{***}$}& {-0.04}\phantom{$^{***}$}& \textcolor{myred}{-0.03}$^{**}$\phantom{$^{*}$}& \phantom{-}{0.06}\phantom{$^{***}$}& {-0.03}\phantom{$^{***}$}& {-0.03}\phantom{$^{***}$}& \phantom{-}{0.07}\phantom{$^{***}$}& {-0.02}\phantom{$^{***}$}& {-0.01}\phantom{$^{***}$}& {-0.04}\phantom{$^{***}$}& {-0.01}\phantom{$^{***}$} \\
Dundee & (\cmark) & {-0.00}\phantom{$^{***}$}& \phantom{-}\textcolor{mygreen}{0.03}$^{**}$\phantom{$^{*}$}& \phantom{-}{0.01}\phantom{$^{***}$}& {-0.00}\phantom{$^{***}$}& \phantom{-}{0.02}\phantom{$^{***}$}& \phantom{-}{0.04}\phantom{$^{***}$}& {-0.00}\phantom{$^{***}$}& \phantom{-}{0.01}\phantom{$^{***}$}& {-0.00}\phantom{$^{***}$}& {-0.00}\phantom{$^{***}$}& {-0.00}\phantom{$^{***}$}& \phantom{-}\textcolor{mygreen}{0.03}$^{*}$\phantom{$^{**}$} \\

Provo &  (\xmark) & \textcolor{myred}{-0.02}$^{***}$& {-0.05}\phantom{$^{***}$}& \phantom{-}{0.07}\phantom{$^{***}$}& {-0.01}\phantom{$^{***}$}& \phantom{-}{0.05}\phantom{$^{***}$}& {-0.02}\phantom{$^{***}$}& {-0.04}\phantom{$^{***}$}& {-0.06}\phantom{$^{***}$}& \phantom{-}{0.10}\phantom{$^{***}$}& {-0.03}\phantom{$^{***}$}& \phantom{-}{0.04}\phantom{$^{***}$}& \phantom{-}{0.03}\phantom{$^{***}$} \\
Dundee &  (\xmark) & \textcolor{myred}{-0.00}$^{**}$\phantom{$^{*}$}& \phantom{-}{0.01}\phantom{$^{***}$}& \phantom{-}{0.00}\phantom{$^{***}$}& {-0.00}\phantom{$^{***}$}& \phantom{-}{0.01}\phantom{$^{***}$}& {-0.00}\phantom{$^{***}$}& {-0.00}\phantom{$^{***}$}& \phantom{-}{0.00}\phantom{$^{***}$}& \phantom{-}{0.02}\phantom{$^{***}$}& \phantom{-}{0.00}\phantom{$^{***}$}& {-0.00}\phantom{$^{***}$}& \phantom{-}{0.02}\phantom{$^{***}$} \\

\midrule
\multicolumn{5}{l}{Renyi Entropy ($\alpha=\sfrac{1}{2}$)} \\
\cmidrule(lr){1-1}

\multicolumn{2}{l}{Brown} & {-0.00}\phantom{$^{***}$}& {-0.00}\phantom{$^{***}$}& \phantom{-}{0.11}\phantom{$^{***}$}& \phantom{-}{0.01}\phantom{$^{***}$}& \phantom{-}{0.00}\phantom{$^{***}$}& \phantom{-}{0.14}\phantom{$^{***}$}& {-0.01}\phantom{$^{***}$}& \textcolor{myred}{-0.01}$^{*}$\phantom{$^{**}$}& {-0.02}\phantom{$^{***}$}& \phantom{-}{0.01}\phantom{$^{***}$}& \phantom{-}{0.00}\phantom{$^{***}$}& \phantom{-}{0.16}\phantom{$^{***}$} \\
\multicolumn{2}{l}{Natural Stories} & {-0.00}\phantom{$^{***}$}& {-0.01}\phantom{$^{***}$}& \phantom{-}\textcolor{mygreen}{0.21}$^{***}$& {-0.01}\phantom{$^{***}$}& \textcolor{myred}{-0.01}$^{*}$\phantom{$^{**}$}& \phantom{-}\textcolor{mygreen}{0.23}$^{***}$& \phantom{-}{0.00}\phantom{$^{***}$}& {-0.00}\phantom{$^{***}$}& {-0.01}\phantom{$^{***}$}& {-0.02}\phantom{$^{***}$}& \textcolor{myred}{-0.01}$^{*}$\phantom{$^{**}$}& \phantom{-}\textcolor{mygreen}{0.21}$^{**}$\phantom{$^{*}$} \\

Provo & (\cmark) & {-0.01}\phantom{$^{***}$}& \phantom{-}\textcolor{mygreen}{0.48}$^{*}$\phantom{$^{**}$}& {-0.03}\phantom{$^{***}$}& {-0.01}\phantom{$^{***}$}& \phantom{-}\textcolor{mygreen}{0.42}$^{*}$\phantom{$^{**}$}& {-0.02}\phantom{$^{***}$}& {-0.02}\phantom{$^{***}$}& \phantom{-}{0.05}\phantom{$^{***}$}& \textcolor{myred}{-0.04}$^{**}$\phantom{$^{*}$}& {-0.02}\phantom{$^{***}$}& \phantom{-}{0.26}\phantom{$^{***}$}& {-0.01}\phantom{$^{***}$} \\
Dundee & (\cmark) & {-0.00}\phantom{$^{***}$}& \phantom{-}\textcolor{mygreen}{0.08}$^{*}$\phantom{$^{**}$}& \phantom{-}\textcolor{mygreen}{0.09}$^{*}$\phantom{$^{**}$}& {-0.00}\phantom{$^{***}$}& \phantom{-}\textcolor{mygreen}{0.07}$^{*}$\phantom{$^{**}$}& \phantom{-}\textcolor{mygreen}{0.10}$^{**}$\phantom{$^{*}$}& \textcolor{myred}{-0.00}$^{***}$& \phantom{-}{0.01}\phantom{$^{***}$}& \textcolor{myred}{-0.00}$^{***}$& {-0.00}\phantom{$^{***}$}& \phantom{-}{0.06}\phantom{$^{***}$}& \phantom{-}\textcolor{mygreen}{0.10}$^{**}$\phantom{$^{*}$} \\

Provo & (\xmark) & {-0.01}\phantom{$^{***}$}& \phantom{-}{0.10}\phantom{$^{***}$}& \phantom{-}{0.04}\phantom{$^{***}$}& {-0.03}\phantom{$^{***}$}& \phantom{-}{0.15}\phantom{$^{***}$}& \phantom{-}{0.01}\phantom{$^{***}$}& \phantom{-}{0.03}\phantom{$^{***}$}& \textcolor{myred}{-0.03}$^{*}$\phantom{$^{**}$}& \phantom{-}{0.04}\phantom{$^{***}$}& {-0.05}\phantom{$^{***}$}& \phantom{-}{0.15}\phantom{$^{***}$}& {-0.01}\phantom{$^{***}$} \\
Dundee & (\xmark) & {-0.00}\phantom{$^{***}$}& \phantom{-}\textcolor{mygreen}{0.04}$^{*}$\phantom{$^{**}$}& \phantom{-}{0.00}\phantom{$^{***}$}& {-0.00}\phantom{$^{***}$}& \phantom{-}\textcolor{mygreen}{0.04}$^{*}$\phantom{$^{**}$}& {-0.00}\phantom{$^{***}$}& {-0.00}\phantom{$^{***}$}& {-0.00}\phantom{$^{***}$}& \phantom{-}{0.01}\phantom{$^{***}$}& \textcolor{myred}{-0.00}$^{***}$& \phantom{-}\textcolor{mygreen}{0.04}$^{*}$\phantom{$^{**}$}& {-0.00}\phantom{$^{***}$} \\

\bottomrule
\\[-10pt]
\multicolumn{6}{l}{
\small
$\bxbase_t = \bxcommon_t \oplus \bxsurp_t \oplus [\renyient]^\intercal$
}
    \end{tabular}
}
\vspace{-5pt}
    \caption{$\deltallh$ (in $10^{-2}$ nats) achieved when predicting RTs after adding budgeting effect predictors on top of a baseline with entropy and surprisal as predictors.}
    \label{tab:delta_llh_budget}
   \vspace{-5pt}
\end{table*}

\subsection{Experiment \#5: Budgeting Effects}

\newcommand{\relu}{\mathrm{r}}
\newcommand{\mathcomment}[1]{\textcolor{mygray}{\texttt{#1}}}

We now analyze budgeting effects.
If RTs are affected by the entropy through a budgeting mechanism, we may expect to see budgeting spillover effects when a reader under-budgets---i.e., when the entropy is smaller than a word's surprisal, causing less time to be allocated to the word than required for processing.
Here, we operationalize \defn{under-budgeting} as any positive difference between surprisal and entropy.
Similarly, we may expect \defn{over-budgeting} to lead to negative spillover-effects, since spending extra time in a word might allow the reader to start going through some of their processing debt (i.e., the still unprocessed spillover effects of that and of previous words). 
We operationalize several potential budgeting effects
as:\looseness=-1%
\begin{subequations}
\begin{align}
    \!& \surpfunc_{t\!-\!1}(\word_{t\!-\!1}) \!-\! \entfunc(\Word_{t\!-\!1}) 
    & \mathcomment{($\Delta$-budget)}  \!\! \\
    \!& \relu(\surpfunc_{t\!-\!1}(\word_{t\!-\!1}) \!-\! \entfunc(\Word_{t\!-\!1}))
    \!\!\!\!
    & \mathcomment{(under-budget)}  \!\! \\
    \!& \relu(\entfunc(\Word_{t\!-\!1}) \!-\! \surpfunc_{t\!-\!1}(\word_{t\!-\!1}))
    \!\!\!\!
    & \mathcomment{(over-budget)} \!\! \\
    \!& |\surpfunc_{t\!-\!1}(\word_{t\!-\!1}) \!-\! \entfunc(\Word_{t\!-\!1})|
    & \mathcomment{($|\cdot|$-budget)} \!\! 
\end{align} 
\end{subequations}
where $\relu(x) = \max(0, x)$.
We then compute the $\deltallh$ of adding these effects as predictors of RT on top of a baseline with the current word's entropy, as well as all four surprisal terms, as predictors
$\bxbase_t\!=\! \bxcommon\oplus\bxsurp\oplus[\renyient]^\intercal$.
Unlike previous experiments, thus, our baseline here already contains the entropy as a predictor.
Further, we show results for eye-tracking datasets both including (\cmark) and excluding (\xmark) skipped words for this and future analyses.

\cref{tab:delta_llh_budget} presents these results.
In short, we do find budgeting effects of word $t\!-\!1$ on RTs in our two analyzed self-paced reading datasets, 
and in Dundee (\cmark).
We do not, however, find them on Dundee (\xmark).
This may imply budgeting effects impact word skipping, but not actual RTs once the word is fixed.
Further, we also find weak budgeting effects of word $t\!-\!2$ in our (\xmark) eye-tracking datasets; these, however, are only significant in Dundee.
We conclude that these results do not provide concrete evidence of a budgeting mechanism influencing RTs, but only of it influencing word skipping instead.
We will further analyze these effects in our discussion section (\cref{sec:discussion}).

\subsection{Experiment \#6: Preemptive Processing} \label{sec:preemptive_processing}

In our analysis of preemptive processing, we will analyze the impact of successor entropy, i.e., $\renyientnext$, on RTs.
While prior work has analyzed this impact, the results in the literature are contradictory.
\cref{tab:delta_llh_successor} presents the results of our analysis.
In short, this table shows that the successor entropy is only significant in Natural Stories.\footnote{We note this is the same dataset previously analyzed by \citet{van-schijndel-linzen-2019-entropy}, who found a significant effect of the successor entropy.\looseness=-1}
In contrast, the current word's contextual entropy is a significant predictor of RTs in $\sfrac{3}{4}$ analyzed datasets, even when added to a model that already has the successor entropy.
Further, while most of our results suggest readers rely on skewed expectations for their anticipatory predictions---i.e., the \renyi{} entropy with $\alpha=\sfrac{1}{2}$ is in general a stronger predictor than Shannon's entropy---the successor Shannon entropy seems more predictive of RTs than the \renyi{}.
Our full model, though, still has a larger log-likelihood when using \renyi{} entropies.
Overall, our results support the findings of \citet{smith2008optimal}, which suggests preemptive processing costs are constant with respect to the successor entropy.
Thus, we conclude preemptive processing is likely not the main mechanism through which $\ent$ affects $\word_t$'s reading times.\looseness=-1

\begin{figure*}[t]
    \centering
    \begin{subfigure}[b]{.5\textwidth}
        \includegraphics[width=\textwidth]{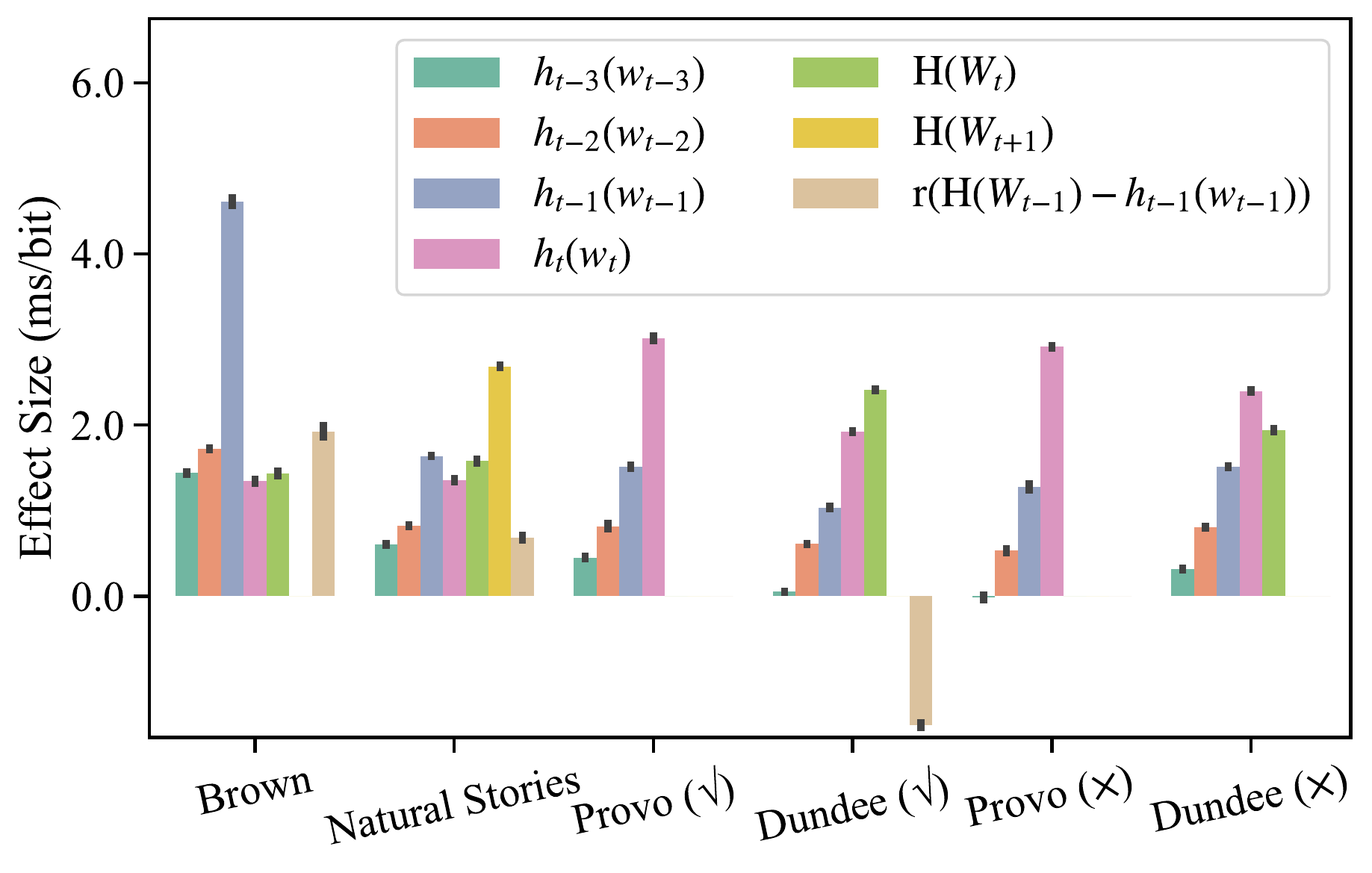}
    \end{subfigure}%
    ~
    \begin{subfigure}[b]{.5\textwidth}
        \includegraphics[width=\textwidth]{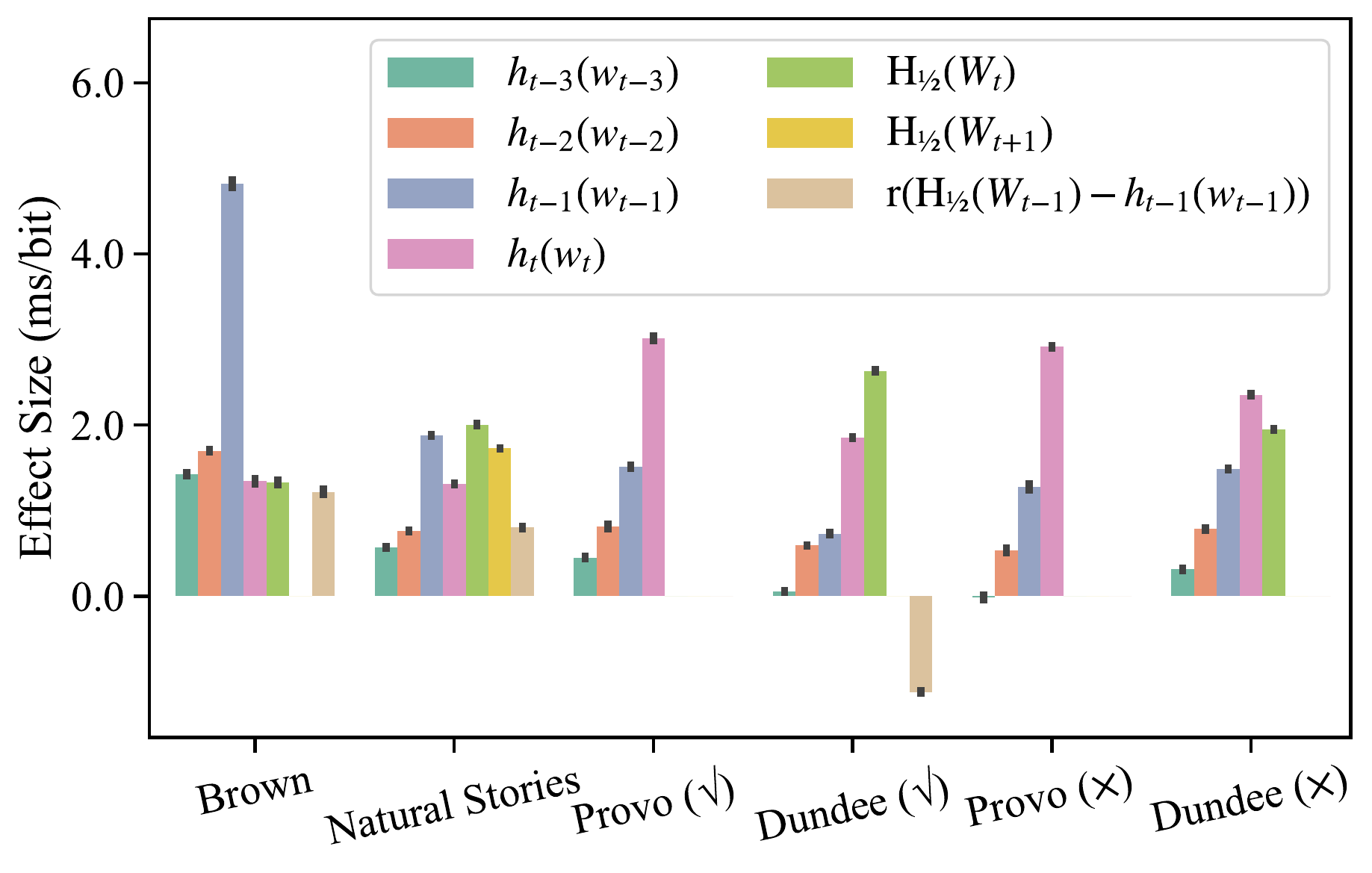}
    \end{subfigure}
    \vspace{-22pt}
    \caption{Size of the learned effects $\gamparams$ for both surprisal and contextual entropy terms when using our best performing models in each dataset: (left) Shannon entropy; (right) \renyi{} entropy with $\alpha=\sfrac{1}{2}$.
    Error bars represent the standard deviation of parameter estimates across the 10 cross-validation folds.
    }
    \label{fig:effect_sizes}
   \vspace{-5pt}
\end{figure*}

\section{Discussion} \label{sec:discussion}

We wrap-up our paper with an overall discussion of results. A key overall finding seen across \cref{tab:delta_llh_entropy,tab:delta_llh_entropy_renyi} is that effects of entropy (expected surprisal) are generally local, i.e., they are clearest and most pronounced on current-word RTs.
On the other hand, the effects of surprisal also show up on subsequent words, e.g., in spillover effects. 
This is consistent with our overall hypothesis that entropy effects capture \textbf{anticipatory} reading behavior.

\begin{table}
    \centering
\resizebox{\columnwidth}{!}{%
    \begin{tabular}{lccc@{\hskip 0in}cc}
\toprule 
         & 
         & \multicolumn{2}{c}{Entropy$^1$} 
         & \multicolumn{2}{c}{Successor Entropy$^2$} 
         \\
         \cmidrule(lr){3-4}
         \cmidrule(lr){5-6}
        &
        & $\emptyset$ $^3$ & $[\renyientnext]^4$
        & $\emptyset$ $^3$ & $[\renyient]^5$
         \\
\midrule
\multicolumn{5}{l}{Shannon Entropy ($\alpha=1$)} \\
\cmidrule(lr){1-1}

\multicolumn{2}{l}{Brown} & \phantom{-}\textcolor{mygreen}{0.15}$^{*}$\phantom{$^{**}$}& \phantom{-}\textcolor{mygreen}{0.14}$^{*}$\phantom{$^{**}$}& \phantom{-}{0.01}\phantom{$^{***}$}& {-0.01}\phantom{$^{***}$} \\
\multicolumn{2}{l}{Natural Stories\,\,\,\,} & \phantom{-}\textcolor{mygreen}{0.89}$^{***}$& \phantom{-}\textcolor{mygreen}{0.44}$^{*}$\phantom{$^{**}$}& \phantom{-}\textcolor{mygreen}{2.27}$^{***}$& \phantom{-}\textcolor{mygreen}{1.83}$^{***}$ \\
Provo & (\cmark) & {-0.06}\phantom{$^{***}$}& {-0.05}\phantom{$^{***}$}& {-0.06}\phantom{$^{***}$}& \textcolor{myred}{-0.06}$^{*}$\phantom{$^{**}$} \\
Dundee & (\cmark) & \phantom{-}\textcolor{mygreen}{0.25}$^{***}$& \phantom{-}\textcolor{mygreen}{0.26}$^{***}$& {-0.00}\phantom{$^{***}$}& {-0.00}\phantom{$^{***}$} \\
Provo & (\xmark) & {-0.01}\phantom{$^{***}$}& \phantom{-}{0.01}\phantom{$^{***}$}& \textcolor{myred}{-0.08}$^{*}$\phantom{$^{**}$}& {-0.06}\phantom{$^{***}$} \\
Dundee & (\xmark) & \phantom{-}\textcolor{mygreen}{0.17}$^{**}$\phantom{$^{*}$}& \phantom{-}\textcolor{mygreen}{0.16}$^{**}$\phantom{$^{*}$}& \phantom{-}{0.02}\phantom{$^{***}$}& \phantom{-}{0.00}\phantom{$^{***}$} \\

\midrule
\multicolumn{5}{l}{Renyi Entropy ($\alpha=\sfrac{1}{2}$)} \\
\cmidrule(lr){1-1}

\multicolumn{2}{l}{Brown} & \phantom{-}\textcolor{mygreen}{0.26}$^{*}$\phantom{$^{**}$}& \phantom{-}\textcolor{mygreen}{0.27}$^{*}$\phantom{$^{**}$}& {-0.01}\phantom{$^{***}$}& {-0.00}\phantom{$^{***}$} \\
\multicolumn{2}{l}{Natural Stories} & \phantom{-}\textcolor{mygreen}{2.04}$^{***}$& \phantom{-}\textcolor{mygreen}{1.52}$^{***}$& \phantom{-}\textcolor{mygreen}{1.95}$^{***}$& \phantom{-}\textcolor{mygreen}{1.44}$^{***}$ \\
Provo &(\cmark) & {-0.04}\phantom{$^{***}$}& {-0.01}\phantom{$^{***}$}& {-0.03}\phantom{$^{***}$}& {-0.00}\phantom{$^{***}$} \\
Dundee &(\cmark) & \phantom{-}\textcolor{mygreen}{0.64}$^{***}$& \phantom{-}\textcolor{mygreen}{0.64}$^{***}$& \phantom{-}{0.00}\phantom{$^{***}$}& {-0.00}\phantom{$^{***}$} \\
Provo &(\xmark) & \phantom{-}{0.32}\phantom{$^{***}$}& \phantom{-}{0.38}\phantom{$^{***}$}& \phantom{-}{0.06}\phantom{$^{***}$}& \phantom{-}{0.12}\phantom{$^{***}$} \\
Dundee &(\xmark) & \phantom{-}\textcolor{mygreen}{0.36}$^{***}$& \phantom{-}\textcolor{mygreen}{0.34}$^{***}$& \phantom{-}{0.03}\phantom{$^{***}$}& \phantom{-}{0.01}\phantom{$^{***}$} \\

\bottomrule
\\[-8pt]
\multicolumn{6}{l}{
$^1$
$\bxmodel_t = \bxbase_t \oplus [\renyient]^\intercal$, 
\,\,\,\,
$^2$
$\bxmodel_t = \bxbase_t \oplus [\renyientnext]^\intercal$, 
} \\ \multicolumn{6}{l}{
$^3$
$\bxbase_t = \bxcommon_t \oplus \bxsurp_t$,
\,\,\,\,
$^4$
$\bxbase_t = \bxcommon_t \oplus \bxsurp_t \oplus [\renyientnext]^\intercal$,
} \\ \multicolumn{6}{l}{
$^5$
$\bxbase_t = \bxcommon_t \oplus \bxsurp_t \oplus [\renyient]^\intercal$, 
}
    \end{tabular}
}
\vspace{-5pt}
    \caption{$\deltallh$ (in $10^{-2}$ nats) after adding the top predictor to a baseline with the predictors in the column.
    All models include surprisal as a predictor.\looseness=-1
    }
    \label{tab:delta_llh_successor}
   \vspace{-5pt}
\end{table}

To make this point more concrete, we plot the values of the parameters $\gamparams$ from our best regressor per dataset in \cref{fig:effect_sizes}---showing the effect of predictor variables not-included in a dataset as zero. 
As the contextual \renyi{} entropy models yield overall higher data log-likelihoods, we focus on them here.
\cref{fig:effect_sizes} shows that---for Brown, Natural Stories and Dundee---not only does the entropy have similar (or stronger) psychometric predictive power than the surprisal, it also has a similar (or stronger) \emph{effect size} on RTs.
In other words, an increase of 1 bit in contextual entropy leads to a similar or larger increase in RTs than a 1 bit increase in surprisal.

\cref{fig:effect_sizes} also shows that in Natural Stories---the only dataset where it is significant---the successor entropy has a larger effect on RTs than the surprisal, and its impact is positive.
This suggests an increase in the next word's entropy may lead to an increase in the current word's RT.
In turn, this could imply that readers preemptively process future words, and that they need more time to do this when there are more plausible future alternatives.
Moreover, we see the successor \renyi{} entropy has a similar (or slightly smaller) effect on RTs than the current word's \renyi{} entropy.
Why the successor entropy is only significant in the Natural Stories dataset is left as an open question.\looseness=-1

\cref{fig:effect_sizes} further shows the effect of over-budgeting on RTs in Brown, Natural Stories, and Dundee.\footnote{While over-budgeting is not a significant predictor in Brown, it leads to slightly stronger models and we add it to this dataset's regressor for an improved comparison.\looseness=-1}
We see that our operationalization of over-budgeting leads to a negative effect on RTs in Dundee (\cmark), but to no effect in Dundee (\xmark).
Together, these results suggest that when a reader over-budgets time for a word, they are more likely to skip the following one.
In Brown and Natural Stories, however, over-budgeting seems to lead to a positive effect on the next word's RT.
As this is only the case in self-paced reading datasets, we suspect this could be related to specific properties of this experimental setting, e.g., a reader's attention could break when they become idle due to over-budgeting RT for a specific word.\looseness=-1

Finally, while we get roughly consistent effect sizes for all predictors in Brown, Natural Stories and Dundee, results are different for Provo.
While we note that Provo is the smallest of our analyzed datasets (in terms of its number of annotated word tokens; see \Cref{tab:data} in \cref{app:data}), this is likely not the whole story behind these different results.
As it is non-trivial to diagnose the source of these differences, we leave this task open for future work.

\section{Limitations and Caveats}

Throughout this paper, we have discussed the effect of anticipation on RTs (and on the reading process, more generally)---where we quantify a reader's anticipation as a contextual entropy.
We do not, however, have access to the true distribution $p$, which is necessary to compute this entropy.
Rather, we rely on a language model $\ptheta$ to approximate it.
How this approximation impacts our results is a non-trivial question---especially since we do not know which errors our estimator is likely to commit.
If we assume $\ptheta$ to be equivalent to $p$ up to the addition of white-noise to its logits,\footnote{I.e., $\ptheta(\cdot \mid \prevwords) \!\propto\! p(\cdot \mid \prevwords)\, e^{\mathcal{N}(0; \sigma^2)}$, where $\mathcal{N}(0; \sigma^2)$ is a normally distributed, $0$-mean noise with variance $\sigma^2$.} for instance, we could have good estimates of the entropy (as the noise would be partially averaged out), while not-as-good estimates of the surprisal (since surprisal estimates would be affected by the entire noise in $\ptheta(\word_t \mid \prevwords)$ estimates).\footnote{This can be made clearer if discussed in terms of the mean squared error of the surprisal and entropy estimates.
The mean squared error of an estimator equals its squared bias plus its variance.
Since contextual entropy is simply the average across surprisals, we should expect the bias term induced by the addition of white noise to be the same in our estimates of both entropy and surprisal. However, the variance term would be larger for surprisals. 
This could bias our analyses towards preferring the entropy as a predictor.}\looseness=-1

We believe this not to be the main reason behind our results for two reasons.
First, if the entropy helped predict RTs simply because we have noisy versions of the surprisal in our estimates, the same should be true for predicting spillover effects, which are also predictable from surprisals. 
This is not the case, however: While the entropy, i.e., $\ent$, helps predict RTs, spillover entropies, e.g., $\entprev$, do not.
Second, even if our estimates are noisy, assuming that this noise is not unreasonably large, a noisy estimate of the surprisal should better approximate the true surprisal than an estimate of the contextual entropy.
Since replacing the surprisal with the contextual entropy eventually leads to better predictions of RTs, this is likely not the only mechanism on which our results rely.\looseness=-1

Another limitation of our work is that we always estimate the contextual entropy and surprisal of a word $\word_t$ while considering its entire context $\prevwords$.
Modeling surprisal and entropy effects while considering skipped words, however, would be an important future step for an analysis of anticipation in reading.
As an example, \citet{van-schijndel-schuler-2016-addressing} show that when a word $\word_{t\!-\!1}$ is skipped, the subsequent word $\word_t$'s RT is not only proportional to its own surprisal (i.e., $\surp$), but to the sum of both these words surprisals (i.e., to $\surp+\surpprev$).
They justify this by arguing that a reader would need to incorporate both words' information at once when reading.
Another model of the reading process, however, could predict that readers simply marginalize over the word in the $(t\!-\!1)^{\text{th}}$ position, and compute the surprisal of word $\word_t$ directly as:\looseness=-1%
\begin{align}
&\log p(\word_{t} \mid \words_{<t\!-\!1}) = \\
&\qquad \log \sum_{\word \in \vocabnoeos} p(\word_t \mid  \words_{<t\!-\!1} \circ \word)\, p(\word \mid \words_{<t\!-\!1}) \nonumber
\end{align}
We leave it to future work to disentangle the effects that using a model $\ptheta$---as well as the effects of skipped words---has on our results.

\section{Conclusion}

This work investigates the anticipatory nature of the reading process. 
We examine the relationship between expected information content---as quantified by contextual entropy---and RTs in four naturalistic datasets, specifically looking at the additional predictive power over surprisal that this quantity provides.
While our results confirm the responsive nature of reading, they also highlight its anticipatory nature.
We observe that contextual entropy has significant predictive power for reading behavior---most reliably on current-word processing---which gives us evidence of a non-trivial anticipatory component to reading.

\appendix

\section{\citeposs{smith2008optimal} Constant Preemptive Processing Effort} \label{app:constant_preprocessing_time}

\newcommand{\preprocessfunc}{\mathrm{pe}}

\begin{proposition}
Assume that the reading times and preprocessing effort (PE) are allocated as follows
\begin{subequations}
\begin{align}
\timefunc(\word \mid \prevwords) &= \frac{\surpfunc_t(\word)}{\log_2 k} \label{eq:reading_time_sl}
    \\
    \preprocessfunc(\word \mid \prevwords) &\propto k^{-\timefunc(\word \mid \prevwords})
     \label{eq:preprocessing_effort_sl}
\end{align}
\end{subequations}
where $\timefunc$ represents reading times here, and $k > 1$ is a free parameter.
Then, the total effort to preprocess all words in the vocabulary, i.e., $\sum_{\word \in \vocab} \preprocessfunc(\word \mid \prevwords)$, is constant.
\end{proposition}
\begin{proof}
By plugging in \cref{eq:reading_time_sl} into \cref{eq:preprocessing_effort_sl}, and summing over the vocabulary, we find preprocessing costs should be proportional to constant.
We show the manipulations below:
\begin{subequations}
\begin{align}
    \sum_{\word \in \vocab} &\preprocessfunc(\word \mid \prevwords)
    \propto \sum_{\word \in \vocab}\! k^{-\frac{\surpfunc_t(\word)}{\log_2 k}}  \\
    &= \sum_{\word \in \vocab}\! k^{\log_k p(\word \mid \words_{< t})}  \\
    &=\  \sum_{\word \in \vocab}\! p(\word \mid \prevwords) = 1  
\end{align}
\end{subequations}
This proves the result.
\end{proof}
\newcommand{\costfunc}{\mathrm{cost}}

\section{A Subword Bound on the \renyi{} Entropy} \label{app:lowerbound}

\newcommand{\vocabsubword}{\mathcal{S}}
\newcommand{\vocabsubwordeos}{\overline{\mathcal{S}}}
\newcommand{\bs}{\boldsymbol{s}}

\newcommand{\subword}{s}
\newcommand{\vocabsingles}{\vocab_{\subword}}
\newcommand{\wordinductive}{\widehat{\subword}}

\begin{theorem}
Let $p$ be a language model with vocabulary $\vocabsubword$, an alphabet of subwords.
Let $\vocabnoeos \subseteq \vocabsubword^*$ be the set of words constructable with subwords drawn from $\vocabsubword$.
Further, assume that, for every $\word \in \vocabnoeos$, there exists a \emph{unique} tokenization of $\word$ into subwords $s_1, \ldots, s_T \in \vocabsubword$ whose concatenation is $\word$, i.e., $\word = s_1 \cdots s_T$.
Due to this uniqueness assumption, we may regard $p$ as either a distribution over $\vocabsubword^*$ or $\vocabnoeos^*$.
Then, for all $\alpha \in \R_{>0}$, we have\looseness=-1
\begin{align}
&\renyientfunc(W_t \mid \Words_{<t} = \prevwords) \\
&\qquad\qquad\quad\geq \renyientfunc(S_t \mid \Words_{<t} = \prevwords) \nonumber
\end{align}
where $S_t$ is an $\overline{\vocabsubword}$-valued random variable that takes on the value of the first subword of the word in $t^{\text{th}}$ position, and $\vocabsubwordeos \defeq \vocabsubword \cup \{\eos\}$.\looseness=-1
\end{theorem}

\begin{proof}
Under the assumption that there exists a unique tokenization of each word $\word \in \vocabnoeos$, we can partition the vocabulary $\vocab$ as follows: $\vocab = \cup_{s \in \overline{\vocabsubword}} \vocabsingles$ where $\vocabsingles$ is the set of words $\word$ which start with subword $s$. 
This allows us to rewrite the R{\'e}nyi entropy as follows:
\begin{subequations}
\begin{align}
&\renyientfunc(W_t \mid \Words_{<t} = \prevwords) \\
    &\, =\frac{1}{1 - \alpha} \log \sum_{w \in \vocab}p(w \mid \prevwords)^\alpha\\
    &\, =\frac{1}{1 - \alpha} \log \sum_{s \in \vocabsubword} \sum_{w \in \vocabsingles}  p(w \mid \prevwords)^\alpha \label{eq:before-inequality} \\
    &\, \geq \frac{1}{1 - \alpha} \log \sum_{s \in \vocabsubword} \left(\sum_{w \in \vocabsingles}  p(w \mid \prevwords)\right)^\alpha \label{eq:after-inequality} \\
    &\, = \frac{1}{1 - \alpha} \log \sum_{s \in \vocabsubword} p(s \mid \prevwords)^\alpha \\
    &\, = \renyientfunc(S_t \mid \Words_{<t} = \prevwords)
\end{align}
\end{subequations}
where we use the fact that $p(s \mid \prevwords) = \sum_{w \in \vocabsingles} p(w \mid \prevwords)$.
The step from \cref{eq:before-inequality} to \cref{eq:after-inequality} holds for $\alpha \in \mathbb{R}_{>0} \setminus \{1\}$ because, for $\alpha > 1$, we have that $\frac{1}{1-\alpha} <0$ and $x^\alpha$ is superadditive for $x \geq 0$.
Similarly, for $0 < \alpha < 1$, we have that $\frac{1}{1-\alpha} > 0$ and $x^\alpha$ is subadditive for $x \geq 0$.\footnote{Superadditivity means $f(a) + f(b) \leq f(a+b)$, while subadditivity means $f(a) + f(b) \geq f(a+b)$. See \url{https://math.stackexchange.com/questions/3736657/proof-of-xp-sub-super-additive} for several simple proofs.}
Finally, for the case that $\alpha = 1$, we can apply 
the chain rule of entropy to write the joint entropy of $W_t$ and $S_t$ in two different ways:%
\begin{subequations}
\begin{align}
    &\entfunc_{1}(W_t, S_t \mid \Words_{<t} = \prevwords) \\ 
    &\qquad\quad=\entfunc_{1}(W_t \mid \Words_{<t} = \prevwords) \\
    &\qquad\quad\qquad\qquad+\underbrace{\entfunc_{1}(S_t \mid W_t, \Words_{<t} = \prevwords)}_{=0} \nonumber\\
   &\qquad\quad=\entfunc_{1}(S_t \mid \Words_{<t} = \prevwords) \\
   &\qquad\quad\qquad\qquad+  \underbrace{\entfunc_{1}(W_t \mid S_t, \Words_{<t} = \prevwords)}_{\geq 0} \nonumber
\end{align}
\end{subequations}
where $\entfunc_{1}(S_t \mid W_t, \Words_{<t} = \prevwords)=0$ because $S_t$ is deterministically derived from $W_t$. 
This implies
\begin{align}
    &\entfunc_{1}(W_t \mid \Words_{<t} = \prevwords) \\ 
     &\qquad\qquad\quad\geq \entfunc_{1}(S_t \mid \Words_{<t} = \prevwords) \nonumber 
\end{align}
This completes the proof for $\alpha \in \mathbb{R}_{> 0}$.\looseness=-1
\end{proof}

\begin{table}
    \centering
\resizebox{\columnwidth}{!}{%
    \begin{tabular}{lrrrr}
        \toprule
        Dataset & \# RTs & \# Words & \# Texts & \# Readers \\
        \midrule
        Brown & 136{,}907 & 6{,}907 & 13 & 35 \\
        Natural Stories & 848{,}767 & 9{,}325 & 10 & 180 \\
        Provo (\cmark) & 225{,}624 & 2{,}422 & 55 & 84 \\
        Dundee (\cmark) & 463{,}236 & 48{,}404 & 20 & 9 \\
        Provo (\xmark) & 125{,}884 & 2{,}422 & 55 & 84 \\
        Dundee (\xmark) & 246{,}031 & 46{,}583 & 20 & 9 \\
        \bottomrule
    \end{tabular}
    }
    \vspace{-5pt}
    \caption{Dataset statistics. \# RTs represents the total number of RT measurements, while \# Words is the number of RTs after averaging across readers.\looseness=-1}
    \label{tab:data}
\end{table}

\section{Datasets }\label{app:data}

Unless otherwise stated, we follow the data preprocessing steps (including cleaning and tokenization) performed by \citet{meister-etal-2021-revisiting}.
We use the following corpora in our experiments:

\newcommand{\datasetparagraphsvspaces}{-3pt}

\paragraph{Brown Corpus.} 
This corpus, first presented in \cite{smith2013-log-reading-time}, consists of moving-window self-paced RTs of selections from the Brown corpus of American English.
The subjects were 35 UCSD undergraduate native English speakers, each reading short (292–902 word) passages.
Comprehension questions were asked after reading, and participants were excluded if their performance was at chance.

\paragraph{Natural Stories.} This corpus is based on 10 stories constructed to have unlikely, but still grammatically correct, sentences.
As it includes psychometric data on sentences with rare constructions, this corpus gives us a broader understanding of how different sentences are processed.
Self-paced RTs on these texts was collected from 180 native English speakers.
We use this dataset's 2021 version, with fixes released by the authors.

\paragraph{Provo Corpus.}
This dataset consists of 55 paragraphs of English text from various sources and genres.
A high-resolution eye tracker (1000 Hz) was used to collect eye movement data while reading from 84 native speakers of American English. 

\paragraph{Dundee Corpus.} We employ this corpus' English portion, which contains eye-tracking recordings (1000 Hz) of 10 native English-speakers.
We drop all measurements from one of these readers (with ID \texttt{sg}), due to them not displaying any surprisal effects as reported by \citet{smith2013-log-reading-time}.
Each participant reads 20 newspaper articles from \emph{The Independent}, with a total of 2,377 sentences.

\begin{figure}
    \centering
    \includegraphics[width=\columnwidth]{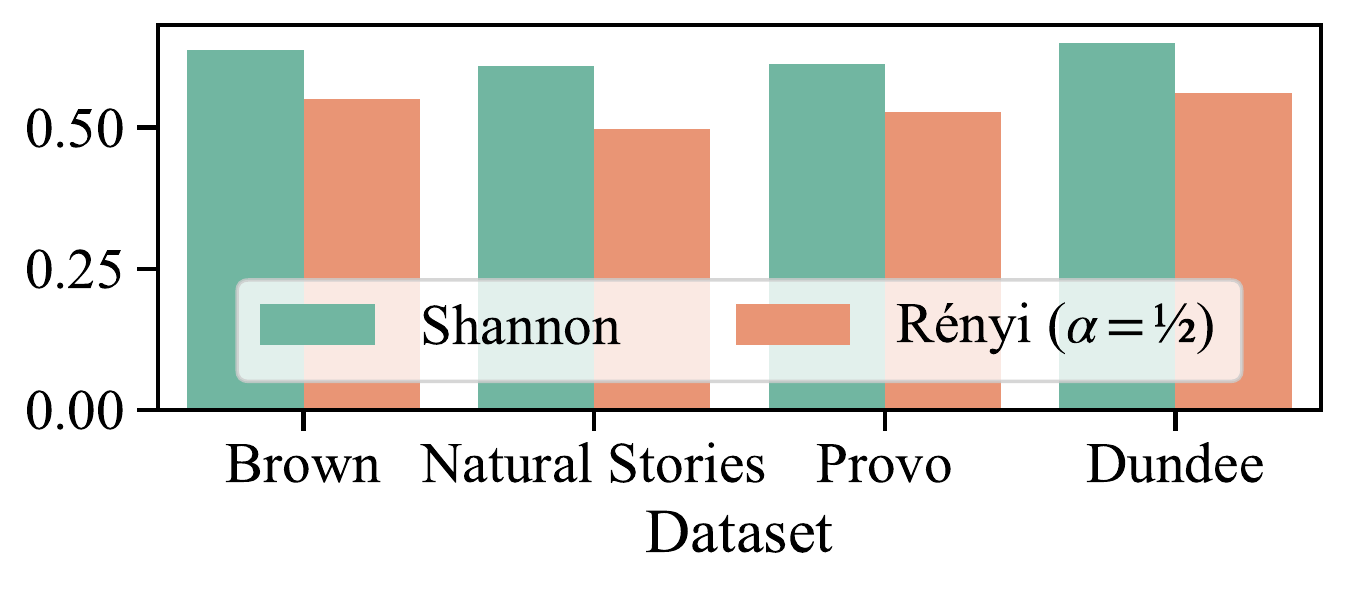}
    \vspace{-25pt}
    \caption{Spearman correlation between the surprisal and contextual entropies.}
    \label{fig:correlations}
    \vspace{-5pt}
\end{figure}

\section{Surprisal vs. Entropy}

Surprisal and contextual entropy are bound to be strongly related, as one is the other's expected value.
To see the extent of their relation, we compute their Spearman correlation per dataset and display it in \cref{fig:correlations}.
This figure shows that these values are indeed strongly correlated, and that Shannon's entropy is more strongly correlated to the surprisal than the \renyi{} entropy with $\alpha\!=\!\sfrac{1}{2}$.
Given that the \renyi{} entropy is in general a stronger predictor of RTs than the Shannon entropy, this finding provides further evidence that our results do not only rely on the entropy ``averaging out'' the noise in our surprisal's estimates.

\section*{Acknowledgments}

We thank Simone Teufel for conversations in early stages of this project.
We also thank our action editor Ehud Reiter, and the reviewers for their detailed feedback on this paper.
Tiago was supported by a Facebook PhD Fellowship.
Clara was supported by the Google PhD Fellowship.
Ethan was supported by an ETH Z{\"u}rich Postdoctoral Fellowship. 
Roger was supported by NSF grant BCS-2121074 and a Newton Brain Science Award.\looseness=-1

\section*{Ethical Considerations}

The authors foresee no ethical concerns with the research presented in this paper.

\bibliography{custom}
\bibliographystyle{acl_natbib}

\end{document}